\documentclass{article}

\PassOptionsToPackage{numbers,compress}{natbib}

\usepackage[main, preprint]{neurips_2026}

\usepackage{newtxtext,newtxmath}
\usepackage[T1]{fontenc}

\DeclareRobustCommand{\VAN}[3]{#2}
\let\VANthebibliography\thebibliography
\def\thebibliography{\DeclareRobustCommand{\VAN}[3]{##3}\VANthebibliography}

\usepackage{graphicx}

\usepackage{amsmath}
\usepackage{amssymb}
\usepackage{amsthm}
\usepackage{hyperref}
\usepackage[capitalise]{cleveref}
\usepackage{bm}
\usepackage{booktabs}
\usepackage{url}
\usepackage{subcaption}
\usepackage{xcolor}

\usepackage{algorithm}
\usepackage{algpseudocode}

\usepackage{tcolorbox}
\tcbset{
  algblock/.style={
    enhanced,
    boxrule=0pt,
    colback=blue!4,
    arc=2pt,
    left=2pt,right=2pt,top=2pt,bottom=2pt,
    boxsep=1pt
  }
}

\definecolor{cNoise}{RGB}{120,120,120}
\definecolor{cGen}{RGB}{30,90,200}
\definecolor{cPos}{RGB}{20,140,70}
\definecolor{cNeg}{RGB}{200,60,60}
\definecolor{cNeg1}{RGB}{160,90,10}
\definecolor{cOp1}{RGB}{160,90,10}
\definecolor{cOp}{RGB}{217,46,129}
\definecolor{cComment}{RGB}{110,110,110}

\definecolor{cChange}{RGB}{0,110,130}

\newcommand{\gen}[1]{\textcolor{cGen}{\ensuremath{#1}}}

\newcommand{\negv}[1]{\textcolor{cNeg1}{\ensuremath{#1}}}

\newcommand{\cmt}[1]{\textcolor{cComment}{\footnotesize\itshape #1}}

\algrenewcommand\algorithmiccomment[1]{\hfill \cmt{#1}}

\DeclareMathOperator{\sech}{sech}
\DeclareMathOperator{\sgn}{sgn}
\DeclareMathOperator{\tr}{tr}

\DeclareMathOperator{\diag}{diag}

\newcommand{\parvec}{\boldsymbol{\theta}}

\newcommand{\mat}[1]{\ensuremath{\bm{\mathsf{#1}}}}
\newcommand{\gmat}[1]{\ensuremath{\bm{#1}}}

\theoremstyle{plain}
\newtheorem{proposition}{Proposition}

\newtheorem{corollary}{Corollary}
\theoremstyle{remark}
\newtheorem*{remark}{Remark}

\title{The Degeneracy Distillery}

\author{%
T. Lucas Makinen\textsuperscript{1,2}\thanks{E-mail: \texttt{tlm41@cam.ac.uk}}\qquad
Deaglan J. Bartlett\textsuperscript{3,4}\qquad 
Niall Jeffrey\textsuperscript{5,6}\qquad 
Benjamin D. Wandelt\textsuperscript{7,8}
\\[0.6em]
{\normalfont\small
\begin{minipage}{0.96\textwidth}
\centering
\textsuperscript{1}Department of Applied Mathematics and Theoretical Physics, University of Cambridge, Wilberforce Rd, Cambridge CB3 0WA, United Kingdom\\
\textsuperscript{2}Imperial Centre for Inference and Cosmology (ICIC), Imperial College London, Prince Consort Road, London SW7 2AZ, United Kingdom\\
\textsuperscript{3}Astrophysics, University of Oxford, Oxford OX1 3RH, United Kingdom\\
\textsuperscript{4}CNRS \& Sorbonne Universit\'e, Institut d'Astrophysique de Paris (IAP), UMR 7095, 98 bis bd Arago, F-75014 Paris, France\\
\textsuperscript{5}Department of Physics and Astronomy, University College London, London WC1E 6BT, United Kingdom\\
\textsuperscript{6}Department of Physics \& King's Institute for Artificial Intelligence, King's College London, London WC2R 2LS, United Kingdom\\
\textsuperscript{7}Department of Physics and Astronomy, Johns Hopkins University, Baltimore, MD 21218, USA\\
\textsuperscript{8}Department of Applied Mathematics and Statistics, Johns Hopkins University, Baltimore, MD 21218, USA
\end{minipage}
}
}

\begin{document}
\maketitle

\begin{abstract}
When two or more parameters or labels produce similar data, they are
\textit{degenerate}, or hard to distinguish.
Degeneracies render both label prediction and inverse problems difficult,
since both machine learning algorithms and probabilistic samplers rely on the
distinguishability of data and its gradients with respect to parameters.
However, identifying degeneracies in physical models or real-world datasets
can be elucidating about the choice of model or the underlying process that
produces the data.
We present the \textit{degeneracy distillery}, a method that
(1) detects and (2) resolves degenerate parameter combinations
(a) automatically and (b) symbolically, from parameter-data
(or parameter-simulation) pairs alone, through estimation and flattening of
the Fisher information matrix. By exploring the information geometry of the likelihood, we characterize degeneracies as an intrinsic property of the physical model, requiring no realised data observation.
We demonstrate our approach on a range of synthetic and real-world problems,
discovering symbolic coordinate transformations that identify the
combinations of parameters of a model which yield independent effects on the
data. The resulting coordinates flatten the Fisher information in expectation \emph{globally},
in contrast to posterior-based methods that flatten only at a single point, and substantially reduce the simulation budget required for downstream neural posterior estimation. In test cases we require up to $10\times$ fewer simulations for posterior estimation at matched validation calibration whilst simultaneously gaining physical insight on the system.

\end{abstract}


\section{Introduction}
\label{sec:intro}

Many scientific models are parametrised by a set of variables, yet their
predictions often depend on these in a highly nontrivial way, with different
parameter variations producing similar or indistinguishable effects.
As a result, variations in individual parameters can produce correlated
effects on observables, with the model responding primarily to particular
combinations of parameters rather than to each independently.
In such situations, we say that there are \textit{degeneracies} in
parameter space.

Identifying these directions is important for both statistical and practical
reasons. From an inference perspective, aligning coordinates with directions
of sensitivity improves conditioning, making likelihood or posterior
exploration and optimisation significantly more efficient. It also provides
physical insight by revealing the combinations of parameters that actually
control observables, rather than relying on arbitrary coordinate choices. By
resolving the true directions along which predictions change, one can
capture parameter-data relationships more precisely and avoid redundancies.
This has direct implications for computational cost: simulations can be
targeted along informative directions rather than wasted along degeneracies
that produce indistinguishable outputs. Finally, such structure enables the
construction of more informed priors, defined in terms of meaningful
parameter combinations rather than poorly aligned individual parameters.

Such behaviour is naturally captured by the Fisher information, which
defines a geometry over parameters that is typically highly anisotropic and
ill-conditioned. In settings with analytic control, one can sometimes
identify appropriate reparametrisations by hand. However, in more complex
or computationally intensive models this becomes infeasible, motivating the
need for systematic, data-driven approaches. In this work, we develop a
method to discover global reparametrisations that regularise this geometry,
yielding coordinates in which the Fisher information is approximately
isotropic across parameter space.

The \textit{degeneracy distillery} maps the information geometry of a system
from parameter-data pairs and discovers a symbolic parametrisation in terms
of the system's original parameters for which the Fisher matrix is as close
to the identity as possible. As outlined in \cref{fig:pipeline}, we achieve
this in three main steps:
(i) Estimate the information geometry everywhere with neural networks;
(ii) Learn neural coordinates that flatten the Fisher geometry and ensure a
unique solution;
(iii) Find nonlinear symbolic expressions that approximate these neural
coordinates and collect like terms.
The output of the distillery is therefore a coordinate transformation that
removes the degeneracies inherent in the original parametrisation.

\paragraph{Contributions.}
{(i) We introduce a three-stage pipeline that infers the Fisher metric
globally over parameter space from simulations alone, and uses it to learn a
symbolic, invertible reparametrisation that flattens the Fisher information.
(ii) We provide ensemble estimators for both the Fisher matrix and the
flattened coordinates, with an alignment scheme that fixes the residual
constant offset and orthogonal symmetries of the flattening loss
(see \cref{app:unique sol}).
(iii) We validate the method on two synthetic problems with known geometry
(Rosenbrock and a one-dimensional Gaussian), where analytic flat coordinates and
geodesic normal
coordinates serve as benchmarks.
(iv) We apply the distillery to four scientific problems---SIR epidemic
dynamics, gravitational-wave inspiral waveforms, weak-lensing cosmology from expensive dark matter simulations, and an industrial heater control problem---each chosen to probe a different regime of
nonlinearity, dimensionality, and prior knowledge.
(v) We demonstrate that the resulting coordinates lower the simulation
budget required by downstream neural posterior estimation while preserving
calibrated coverage.}

We briefly review the related literature in \cref{sec:related}, then
describe the principles of information geometry in \cref{sec:theory}.
We detail the degeneracy distillery pipeline in \cref{sec:method}, before
validating it against synthetic examples and applying it to scientific
problems in \cref{sec:experiments}. We discuss limitations and conclude in
\cref{sec:discussion}. Full
experimental, architectural, and theoretical details are deferred to the
appendices.

\begin{figure}
    \centering
    \includegraphics[width=\textwidth]{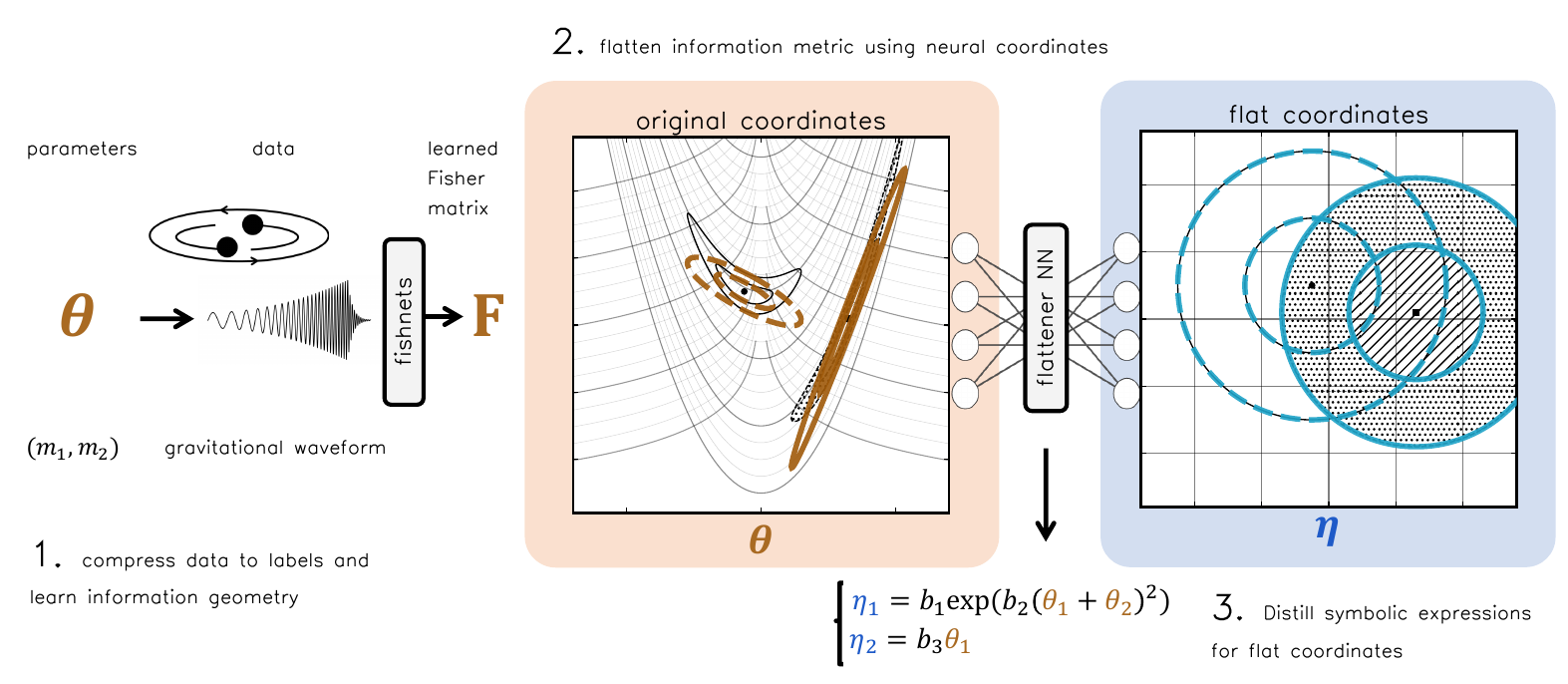}
    \caption{Degeneracy distillery pipeline in three steps.
    (1)~Parameter-data pairs $(\negv{\theta}, \textbf{x})$ are passed to an
    ensemble of Fishnet networks to learn approximations to the Fisher
    matrix at each parameter value.
    (2)~A ``flattener'' network is trained so that its Jacobian maps the
    learned metric to the identity as a function of $\negv{\theta}$.
    (3)~Symbolic regression is performed on each output $\gen{\eta}$
    coordinate to obtain short, nonlinear, degeneracy-resolving expressions.
    }
    \label{fig:pipeline}
\end{figure}

\section{Related Work}
\label{sec:related}

\paragraph{Sloppy models.}
``Sloppy'' models describe a common feature of multiparameter nonlinear
systems: predictions depend strongly on a few parameter combinations while
remaining insensitive to many others. This is reflected in a Fisher
information matrix with eigenvalues spanning many orders of magnitude,
producing a highly anisotropic parameter geometry and a ``hyperribbon''
model manifold \citep{Gutenkunst_2007,Transtrum_2011,Transtrum_2015}.
\citet{Gutenkunst_2007} interpret the Fisher information as a Riemannian
metric, using its eigenstructure to distinguish stiff and sloppy directions
and explain how accurate predictions coexist with poorly constrained
parameters. Exact or near degeneracies are handled explicitly, either as
directions of vanishing sensitivity or via reductions such as the manifold
boundary approximation method (MBAM), which collapse these directions at
manifold boundaries \citep{Transtrum_2014}.

\paragraph{Local conformal autoencoders and intrinsic coordinates.}
The LOCA method \citep{Peterfreund_2020} and \citet{Evangelou_2021} also
construct intrinsic coordinates that remove redundancies from nonlinear
observation maps, but without explicit use of likelihood-based geometry.
LOCA enforces local whitening in data space, normalising the pull-back
metric and producing coordinates invariant (up to conformal factors) under
smooth observation transformations. \citet{Evangelou_2021} instead identify
parameter combinations that determine outputs by analysing the
parameter-to-observable map, separating predictive directions from level
sets. Both are data-driven and focus on removing non-identifiable directions
rather than flattening an information metric. They yield coordinates adapted
to the model manifold, whereas our approach seeks to regularise the Fisher
geometry over parameter space.

\paragraph{Geometric variational inference and posterior reparametrisation.}
Geometric variational inference (geoVI) \citep{Frank_2021} also uses the
Fisher information, but locally: it constructs posterior-specific coordinates
that flatten the geometry near the posterior, enabling Gaussian
approximations. Our approach instead provides a global, model-intrinsic
reparametrisation independent of data.
Similarly, \citet{Dacunha_2022} learn global density models to derive the local
Fisher geometry and decompose parameters into constrained and unconstrained
directions. This is also posterior-specific, whereas our method is global;
additionally, their coordinates are neural, while ours use explicit analytic
forms via symbolic regression.

\paragraph{Symbolic regression and information-geometry priors.}
{The symbolic component of our pipeline builds on \textsc{operon}
\citep{Burlacu_2020,Kronberger_2024} and on minimum-description-length
selection criteria for symbolic regression
\citep{Bartlett_2024_ESR,Bartlett_2023_priors}. Ours is, to our knowledge,
the first method to combine simulation-based Fisher estimation with a
global, symbolic flattening of the resulting metric.}

\section{Theoretical Background}
\label{sec:theory}

Consider a model that produces $n_\textbf{x}$ data points $\textbf{x}$,
controlled by $n_\theta$ parameters $\negv{\theta}$. The likelihood,
$p(\textbf{x} | \negv{\theta})$, describes the probability distribution of
data that is produced at a particular parameter value. A higher likelihood
value indicates a better description of the model, via $\negv{\theta}$, of
the data. The Kullback--Leibler (KL) divergence between two points,
$\negv{\theta}_1$ and $\negv{\theta}_2$, describes how much more likely
the choice of $\negv{\theta}_1$ describes the data than $\negv{\theta}_2$:
\begin{equation}
    D_{\rm KL}(\negv{\theta}_1 \,||\, \negv{\theta}_2) =
    \int {\rm d}\textbf{x}\ p(\textbf{x} | \negv{\theta}_1)
    \left(\log p(\textbf{x} | \negv{\theta}_1) - \log p(\textbf{x} | \negv{\theta}_2)\right).
\end{equation}
This does not satisfy the triangle inequality, and is not symmetric, making it a poor choice of distance metric. However,
\citet{rao1945information}'s insight came from considering a very small
deviation from a point, $\negv{\theta} + \delta\negv{\theta}$. Here, the KL
divergence reduces to:
\begin{equation}\label{eq:rao-KL}
    D_{\rm KL}(\negv{\theta} \,||\, \negv{\theta} + \delta\negv{\theta}) =
    \frac{1}{2}F_{ab}\,\delta\negv{\theta}^a\,\delta\negv{\theta}^b + \mathcal{O}(\delta\negv{\theta}^3),
\end{equation}
where the \textit{Fisher Information Matrix}, $\mat{F}$, is defined as:
\begin{equation}
    F_{ab}(\negv{\theta}) = -\int {\rm d}\textbf{x}\ p(\textbf{x} | \negv{\theta})\,
    \frac{\partial}{\partial \negv{\theta}^a}\frac{\partial}{\partial \negv{\theta}^b}
    \log p(\textbf{x} | \negv{\theta}),
\end{equation}
where we use the Einstein summation convention. The Fisher Matrix is
positive semi-definite, and transforms properly under smooth diffeomorphisms, making it an ideal candidate for a Riemannian metric.
This metric underpins the underlying \textit{information geometry} of the
likelihood \citep{nielsen_elementary_2020, rao1945information}, and is controlled by the parameters $\negv{\theta}$ via the
expectation over the data generated by the model. The Fisher matrix is not
invariant to the way we choose to parametrise our system. Suppose we have
some parametrisation in terms of the parameters $\{\negv{\theta}^a\}$, for
which we have the Fisher matrix $\mat{F}$. If we now change to a new
parametrisation in terms of $\{\gen{\eta}^a\}$ then we have the new Fisher
matrix
\begin{equation}\label{eq:metric_transform}
    F_{ab}^\prime =
    \frac{\partial \negv{\theta}^c}{\partial \gen{\eta}^a}
    \frac{\partial \negv{\theta}^d}{\partial \gen{\eta}^b}\,F_{cd}.
\end{equation}
If all the parameters of a system acted completely independently, then we
would have a diagonal Fisher matrix. However, in general, this is not the
case, since some of the parameters are \textit{degenerate}: changing the
value of one of them changes the data in a similar way to changing another.
In terms of the Fisher information, a statistical degeneracy is a
parameter-space direction $\delta\negv{\parvec}$ with vanishing or very small
curvature, i.e.\ $\delta\negv{\parvec}^\top\mat{F}\,\delta\negv{\parvec} \approx 0$, where
$\mat{F}$ has small eigenvalues.

Ideally, one would find a parametrisation in $\gen{\bm{\eta}}$ that makes the Fisher
matrix diagonal everywhere in parameter space, and hence all degeneracies
would be removed.
It is even more desirable to have a parametrisation in which $\mat{F}$ is
equal to the identity, such that a change of unity in any of the
$\gen{\bm{\eta}}$ parameters has a similar effect on the likelihood. 

Utilising the analogy with geometry, this corresponds to finding a coordinate
system (parametrisation) in which the metric (Fisher matrix) is as close to
the identity matrix as possible.
In some special cases, it is possible to do this exactly everywhere, which
in our geometry analogy corresponds to a system with zero curvature (the
Riemann tensor is zero \citep{lee_introduction_2018}).
In general, however, there will be some curvature, and thus we can only aim
to make our Fisher matrix as flat as possible, even if we cannot completely
remove all degeneracies or curvature.
Our goal is therefore to find such a (symbolic) coordinate transformation
automatically from data--parameter pairs, for which the average deviation
from flatness across the prior range of the parameters is minimised.

\section{Method}
\label{sec:method}

The pipeline is summarised in \cref{fig:pipeline}.
We separate the construction into three stages so that each can be assessed independently.

\paragraph{1.~Mapping information geometry with Fishnets.}
Given a set of parameter-data pairs, $\{(\negv{\bm{\theta}}, \textbf{x})\}$, we
first wish to estimate the Fisher matrix at all points in parameter space.
Traditionally, this is done by running new simulations at perturbed values
of the parameters and using finite differences, or by leveraging
auto-differentiation. The former can be costly, and the latter is only
possible when one has a differentiable pipeline.
To circumvent these issues, we use the Fishnets formalism
\citep{Makinen_fishnets_2023}, which we briefly review here. Given the data vector $\textbf{x}$, we train a compression network to predict
estimates of the parameters, $\hat{\bm{\theta}}(\textbf{x})$, as well as the
Fisher matrix, $\mat{F}(\textbf{x})$. Both functions are controlled by a set
of model weights, $\varphi$. To ensure a positive-definite matrix, the
network does not directly output $\mat{F}$, but instead its Cholesky
decomposition, $\mat{L}$, such that $\mat{F} = \mat{L}\mat{L}^{\rm T}$.

Our optimisation objective is
\begin{equation}
    L_{\rm fish}(\varphi) =
    \tfrac{1}{2}\!\left(\negv{\bm{\theta}} - \hat{\bm{\theta}}(\textbf{x})\right)^{\!\top}
    \mat{F}(\textbf{x})
    \left(\negv{\bm{\theta}} - \hat{\bm{\theta}}(\textbf{x})\right)
    - \tfrac{1}{2}\log\det \mat{F}(\textbf{x}).
\end{equation}
This objective is minimised 
over batches of pairs $(\negv{\bm{\theta}}, \textbf{x})$:
\begin{equation}
\label{eq:lossfn}
    \min_{\varphi} L_{\rm fish}
    = \min_{\varphi}\int\! {\rm d}\textbf{x}\,{\rm d}\negv{\parvec}\,
    L_{\rm fish}\,p(\textbf{x}, \negv{\parvec}),
\end{equation}
which can be seen as a variational minimisation over parameter space of
\cref{eq:rao-KL}. We show in \cref{app:min_post_cov} that minimising this
objective yields the posterior mean and (inverse) covariance, which
asymptotically approaches the underlying Fisher information.

In practice, the Fishnets prediction is a noisy estimator of the ``true''
underlying Fisher information. To compensate, we train an ensemble of
networks. For each member $i$, we re-evaluate \cref{eq:lossfn} on a
validation set, obtaining a value $L_i$, and form a validation-weighted
mean and uncertainty
\begin{equation}
    \bar{\mat{F}} = \sum_{i} w_i\,\mat{F}^{(i)},
    \quad
    \delta F_{ab} = \sqrt{\sum_{i} w_i\!\left(F_{ab}^{(i)} - \bar{F}_{ab}\right)^{\!2}},
    \quad
    w_i = \frac{\exp(-L_i)}{\sum_{j}\exp(-L_{j})},
\end{equation}
on the individual entries of the Fisher matrix.
We do not consider any covariance between the elements of $\mat{F}$.
To ensure sufficient diversity, we vary hidden size over a wide uniform
range and choose activation functions at random from a library of options
(\cref{app:architecture}). We find that $10$--$20$ networks
produce Gaussian errors on the entries of $\mat{F}$.

\paragraph{2.~Neural coordinates.}
Once we have an estimate for the Fisher matrix at all points in parameter
space, we wish to find a new set of coordinates for which the Fisher is
close to the identity everywhere.
We parametrise this coordinate transformation with a neural network which
takes the original coordinate values as input and outputs the transformed
coordinates, $\gen{\bm{\eta}} = \gen{\bm{\eta}}_{\rm NN}(\negv{\bm{\theta}})$.
One can take the derivative of this network with respect to the inputs to
obtain its Jacobian
$\mat{{J}} = [\partial \gen{\bm{\eta}}/\partial \negv{\bm{\theta}}]$.
We use the pseudo-inverse of $\mat{\mathcal{J}}$ to transform $\bar{\mat{F}}$
to its value in the new coordinates,
$\bar{\mat{F}}^\prime$, using \cref{eq:metric_transform}.

To train the network, we compute the Frobenius norm between
$\bar{\mat{F}}^\prime$ and the identity:
\begin{equation}\label{eq:flat_loss}
    \ell(\mat{A}) =
    \big\|{A}_{ab} - \delta_{ab}\big\|_{\mathcal{F}}^{2}
    \equiv \tr\!\left((\mat{A}-\mat{1})(\mat{A}-\mat{1})^{\rm T}\right).
\end{equation}
We additionally penalise the inverse, so that the loss is
$L(\gen{\bm{\eta}}) = \ell(\bar{\mat{F}}^\prime) + \ell(\bar{\mat{F}}^{\prime\,-1})$. This loss is summed over all training points, averaging over the prior from
which the simulations were drawn. Instead of locally flattening our
coordinates around a fiducial point, this corresponds to making them
approximately flat over all of space. As with estimating the Fisher information, we fit an ensemble of networks
to this coordinate transformation to compute uncertainties.
We begin by fitting a single network to the coordinate transformation using
$\bar{\mat{F}}$ as our Fisher estimate.
This network is then copied multiple times, and each copy is fine-tuned to
a different Fisher estimate, $\mat{F}^{(i)}$, from the earlier ensemble, to
obtain $\gen{\bm{\eta}}^{(i)}$.
After ensuring that each of these estimates is aligned (see
\cref{app:unique sol}), we compute the mean and uncertainty
\begin{equation}
    \bar{\gen{\eta}}^a = \sum_{i} w_i\,{\gen{\eta}^{(i)}}^a,
    \quad
    \delta \gen{\eta}^a =
    \sqrt{\sum_{i} w_i\!\left({\gen{\eta}^{(i)}}^a - \bar{\gen{\eta}}^a\right)^{\!2}},
\end{equation}
on each transformed coordinate, with the same weighting as before.
We provide several architecture options for the flattening operator, but
use an inverse-conditioned, fully-connected MLP by default (see
\cref{app:architecture}).

\paragraph{3.~Symbolic coordinates.}
For increased interpretability and ease of use, we wish to find a symbolic
approximation to the coordinate transformation learned by the network.
We apply symbolic regression \citep{Kronberger_2024} with \textsc{operon}
\citep{Burlacu_2020} separately for
each output coordinate (and thus assume
sufficiently independent errors for each component), using a Gaussian loss
function with the previously calculated mean and uncertainty. {Crucially, once trained, the neural map $\gen{\bm{\eta}}(\negv{\bm{\theta}})$ is decoupled from the original simulator: we therefore draw a fresh, dense set of parameters from the prior and pass them through the flattening ensemble to construct the augmented $(\negv{\bm{\theta}}, \bar{\gen{\bm{\eta}}}, \delta\gen{\bm{\eta}})$ training set used by SR. In our experiments, this augmentation was the most consistent way to improve
symbolic-regression convergence--e.g.\ on Rosenbrock we use ${\sim}2{,}000$ prior draws against ${\sim}250$ Fishnet validation points (\cref{app:experiments}).}  {We use the default \textsc{operon} configuration described in
\cref{app:architecture}: an offspring-generator with maximum equation length
$25$, depth $10$, $10$ optimiser iterations per candidate, the default operator set
\texttt{\{add, mul, pow, sqrt, square, exp, logabs\}}, and a per-component time
budget of $120$\,s.} Given a final Pareto front of the most accurate models at each complexity,
we select our final equation using a formulation of the
minimum-description-length (MDL) principle relevant to symbolic regression
\citep{Bartlett_2024_ESR}. This single objective penalises functions which
are poor fitting, highly complex, or have many or finely tuned parameters,
and is an approximation to the logarithm of the probability of the
function given the data, under a specific prior on the equation's form
\citep{Bartlett_2023_priors}.

For comparison, we also leverage a ``flatness'' metric in the form of
\cref{eq:flat_loss}. To select expressions component-wise we replace the
Jacobian column $\partial \eta_i/\partial \theta$ calculated from the
neural network with the gradient computed from the proposed symbolic
expression, and select those expressions which result in the minimum change
in loss with respect to the neural coordinates. This criterion does not
penalise expression length, but most often yields a flatter geometry.
All models undergo a pruning procedure to remove unnecessary terms, as
outlined in \cref{app:post_process}.

\section{Experiments}
\label{sec:experiments}
We apply our method to both synthetic and scientific applications. In both settings we demonstrate both the physical insights (learned coordinates) and efficiency gains degeneracy resolution brings to Bayesian inference problems, in particular neural posterior estimation.

\subsection{Synthetic validation}
\label{sec:synthetic_validation}

\begin{figure}
    \centering
    \includegraphics[width=\linewidth]{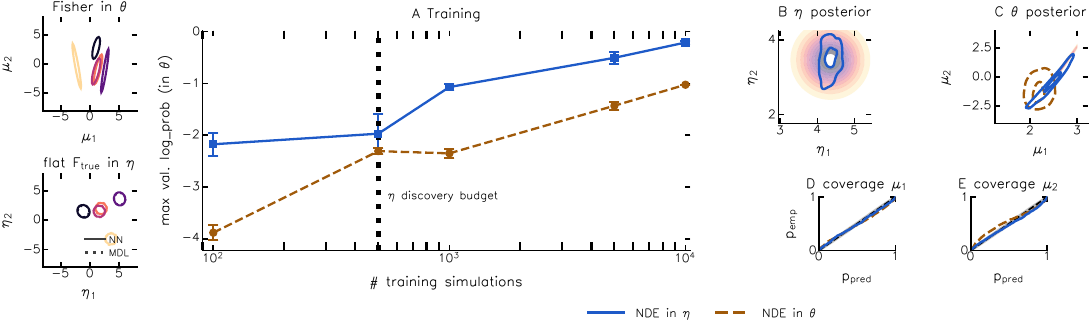}
    \caption{Rosenbrock validation case. \textit{Left}: Learned Fisher curvature (top) is flattened by learning $\gen{\eta}$ (bottom). \textit{Right}: Knowing the symbolic, flat geometry A) makes learning posteriors easier even with increasing simulation budget B) yields faithful posteriors in target $\negv{\theta}$ coordinates and C) is better calibrated than learning in degenerate $\negv{\theta}$ space. 
    }
    \label{fig:rosenbrock-main}
\end{figure}

\paragraph{Rosenbrock.}
We consider the two-dimensional Rosenbrock function: a notoriously
difficult, non-convex test problem whose information geometry is known and
can be analytically mapped to the identity using ``flat'' coordinates
(\cref{app:rosenbrock}). Using data generated as described in
\cref{app:rosenbrock}, we display the results of our three-step procedure
in \cref{fig:rosenbrock-main}, where our original coordinates are
$\bm{\theta} = (\mu_0, \mu_1)$ and the transformed coordinates are
${\bm{\eta}} = (\nu_0, \nu_1)$.

\textbf{Results.} The Fishnet ensemble produces an unbiased and well-calibrated estimate of
the true underlying information geometry of the Rosenbrock posterior,
illustrated via (a) Fisher Cholesky factors and (b) trace-normalised Fisher
matrix shape $M_{ij} = F_{ij}/\tr\mat{F}$ (Fig.~\ref{fig:rosen-sampling-dist}).
The flattening operation
produces ensemble-computed neural coordinates $\eta$ and residuals
$\delta\eta$ that are consistent with $\mathcal{N}(0,1)$, ideally suited
for the symbolic-regression MDL selection criterion. The symbolic regression
and post-processing step consistently recovers the true expressions,
$\eta = [b_0\mu_0^2 + b_1\mu_1,\; b_2\mu_1]$. Knowing amortised coordinates makes downstream inference in the Rosenbrock setting more simulation efficient. We train an ensemble of NPE networks at different simulation budgets on an exacerbated Rosenbrock geometry ($\Sigma={\rm diag}(4^2,1^2)$). Using the learned coordinates we reconstruct the true posteriors at all simulation budgets with 1) better calibration and 2) up to $10\times$ fewer simulations than NPE in raw $\theta$.

\paragraph{Gaussian with unknown variance.}
We now turn to an example which cannot be flattened exactly at all points
in parameter space.
We consider a one-dimensional Gaussian, parametrised by mean $\mu$ and
standard deviation $\sigma$.
Taking draws from this distribution, the Fisher matrix is
$F_{ab} = n_{\rm d} \diag(1/\sigma^2, 2/\sigma^2)$.
As discussed in \cref{app:gauss_normal_coords}, the Riemann tensor for a
metric with $g_{ab}=F_{ab}$ is non-zero in this case
(\cref{eq:gauss_riemann_tensor}), confirming that one cannot flatten the
Fisher matrix everywhere.
Hence, we wish to find a choice of parameters which makes the Fisher as
flat as possible, even if one cannot make it identical to the identity
everywhere.
Given a fiducial point, it is possible to enforce only quadratic deviations
from flatness through the use of geodesic normal coordinates. These are
constructed in \cref{app:gauss_normal_coords}, and are used as a benchmark:
when averaged over the prior, our method should construct coordinates with
a lower loss than these analytic ones, although close to the fiducial point
the analytic ones must perform better. We choose the fiducial point to be $(0.0, 1.0)$. We generate $10\,000$ simulations via $\mu \sim \mathcal{U}[-5,5]$ and
$\sigma \sim \mathcal{U}[0.2, 7.0]$, and generate
$x \sim \mathcal{N}(\mu, \sigma)$ with $\dim(x) = 50$. We adopt the same
network setup calibrated on the Rosenbrock example. 

\textbf{Results.}
We obtain (i) neural coordinates and (ii) two sets of symbolic coordinates
selected according to (a)~MDL and (b)~Frobenius criteria. We compare the
``flatness'' of our coordinates against \textit{true}, analytic Fisher
matrices as a function of geodesic distance $\beta$ (see
\cref{app:gauss_normal_coords}), computed over test data. For comparison,
we compute geodesic normal coordinates around the fiducial point
$\theta^\star = (\mu^\star, \sigma^\star) = (0, 1)$, as well as ``ad-hoc''
coordinates typically used for Gaussian parameter inference,
$\eta^{\rm ad\text{-}hoc} = (\mu/\sigma,\,\log\sigma)$.
\cref{fig:mu-sigma-heater-scaling} shows that geodesic coordinates
obtain perfect flattening close to the fiducial point ($\beta\approx 0$),
but fail exponentially quickly further away. Both neural and symbolic coordinates
consistently flatten the true geometry across the parameter space.

\begin{figure}[t]
  \centering
  \includegraphics[width=
  \linewidth]{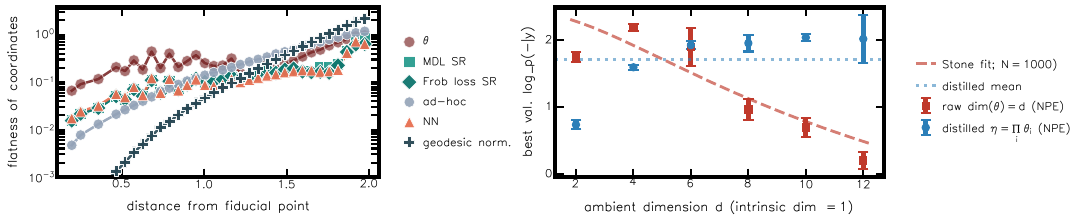}
  \caption{\textit{Left:}~Neural and symbolic regression components outperform
  (i)~hand-picked (ad-hoc) coordinates and (ii)~flatten geometry far from
  the fiducial point \cref{eq:beta_from_mu_sigma},
  without the need for an analytic Fisher matrix to derive geodesic normal
  coordinates. \textit{Right:}~Inference validation $\log p(\cdot | y)$ shows little degredation for NPE estimating distilled $\eta = \prod_i \theta_i \in \mathbb{R}$ for the
  chain-product heater simulator with
  intrinsic dim $= 1$, while raw NPE in $\theta \in \mathbb{R}^{d}$
  degrades with $d$ at fixed training budget $N$. }
  \label{fig:mu-sigma-heater-scaling}
\end{figure}

\subsection{Scientific applications}
\label{sec:applications}

\paragraph{Industrial Heater control.}
\label{sec:heater}
We construct a minimal synthetic experiment to demonstrate that the distillery can perform rank-deficient \emph{dimensionality reduction}, improving NPE scaling.  Here we explore a rank-deficient diagnostic case, not a global flattening case. A first-order resistive-heater
plant is driven by an applied voltage and current $\theta = (V, I)$, with observable 20-point time series of the temperature step response (see e.g. \cite{Incropera2011}),
\begin{equation}
  y(t) \;=\; (V \cdot I)\,(1 - e^{-t/\tau}) \;+\; \varepsilon(t),
  \qquad \varepsilon(t) \overset{\text{iid}}{\sim}
         \mathcal{N}(0, \sigma^{2}).
  \label{eq:heater_simulator}
\end{equation}
The plant dynamics depend on $(V, I)$ \textit{only} through the dissipated power
$P = V \cdot I$, so the Fisher information matrix
$F_\theta(\theta) = (\|k\|^{2}/\sigma^{2})\,(\theta_2, \theta_1)^{\top}
(\theta_2, \theta_1)$ is structurally rank-1 \emph{everywhere}: $\theta$
parametrises a 2-D box but the data manifold is exactly 1-D. This is a common occurrence in industrial control problems where two physical dials multiply together to produce a single output. 

\textbf{Results.} Running our method on $N_{\text{sim}} = 10^{3}$ simulations with a
uniform prior $\theta_i \sim \mathcal{U}[1, 2]$ recovers two affine
coordinates whose symbolic form is $\eta_{\text{nuis}} \;\approx\; 0.10\,(V - I) + 0.4$ and $\eta_{\text{iden}} \;\approx\; -0.6\,(V + I) + 6.0$. The pipeline therefore
\emph{automatically} identifies the 1-D submanifold of identifiable
parameters and the orthogonal nuisance direction from the data alone (see App. \cref{app:heater} for details). \textbf{Scaling Impact.} Density estimation becomes increasingly difficult in high dimensions, requiring $N(d, \varepsilon) \;\sim\; \varepsilon^{-(2\beta + d)/(2\beta)}$ simulations to reach an error tolerance $\varepsilon$ (Stone's relation; \cite{stone1980}). 
If the data manifold
has intrinsic dimension $d^{*} < d$, training on the distilled
coordinates costs only $N(d^{*}, \varepsilon)$, giving an
exponential-in-$(d-d^*)$ saving (see App \ref{app:heater}).

Explicit symbolic (or neural) distillation of the \textit{intrinsic} parameter combinations delivers this exponential gain in density estimation by effectively shrinking the parameter space. 

\textbf{Empirical scaling sweep.}
To make this explicit in the context of neural posterior estimation (NPE; \cite{papamakarios2021normalizing,pydelfiAlsing_2019}), we extend the heater problem to a chain-product simulator
\begin{equation}
  y(t) \;=\; \Big(\prod_{i=1}^{d} \theta_i\Big)\,
             (1 - e^{-t/\tau}) \;+\; \varepsilon(t),
  \qquad \theta_i \sim \mathcal{U}[1, 2],
\end{equation}
in which the intrinsic dimension is fixed at one for every $d$.
Figure~\ref{fig:mu-sigma-heater-scaling} shows the best validation
log-probability of two neural density estimators
trained on a fixed budget of $N_{\text{sim}} = 10^{3}$ simulations: the
raw flow with target $\theta \in \mathbb{R}^{d}$ and the distilled flow
with target $\eta = \prod_i \theta_i \in \mathbb{R}$. The raw flow's
inference quality degrades monotonically with $d$, tracing Stone's relation, while the distilled
flow is, by construction, insensitive to $d$. The gap widens at the rate
predicted by Eq.~\eqref{eq:distillation_gain}, despite the underlying
observed information being identical at every $d$.

\begin{figure}
    \centering
    \includegraphics[width=\linewidth]{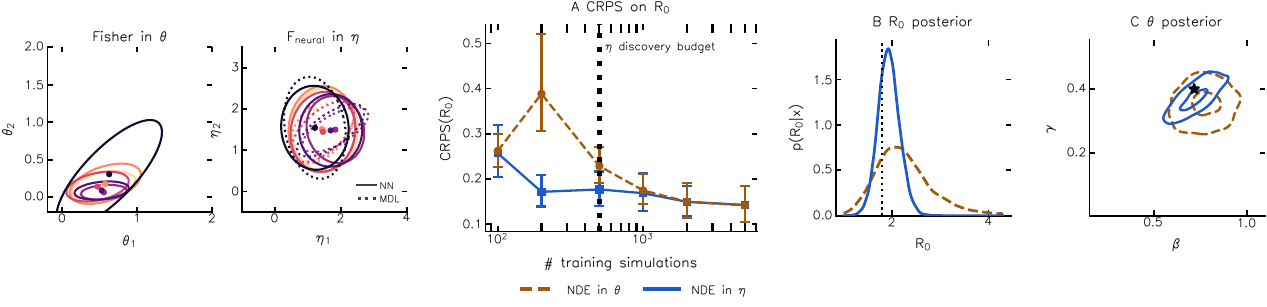}
    \caption{Epidemiology example. Distilled coordinates flatten learned curvature (\textit{left}). Amortised $\eta$ make downstream NPE more efficient in $R_0=\beta/\gamma$, measured by CRPS at fixed simulation budgets and higher noise levels (\textit{middle}). \textit{Right:} Posteriors sampled in $\eta$ capture $R_0$ with higher fidelity and trace the degeneracy direction in $(\beta,\gamma)$ (shown for the $N_{\rm sim}=500$ case).}
    \label{fig:sir_results}
\end{figure}

\paragraph{Epidemic infection rates (SIR).}
\label{sec:epidemiology}
The Susceptible--Infected--Recovered model is a standard model of epidemic dynamics
\citep{kermack1927contribution,hethcote2000mathematics}. Epidemic-outbreak
dynamics can be modelled using coupled ordinary differential equations
describing susceptible, infectious, and recovered populations, governed by
infection ($\beta$) and recovery ($\gamma$) rates. A disease becomes epidemic when its basic reproduction number
$R_0 = \beta/\gamma > 1$, corresponding to the case where one infected
person infects more than one other person on average. 
\textbf{Data.}
Epidemiologists usually only see noisy estimates of infection numbers over
time. Here we ask whether a nonlinear reparameterisation of
$(\beta, \gamma)$ can be learned from noisy simulations alone. We focus on the \emph{epidemic regime}, retaining simulations with
$R_0 \geq 1.15$ and discarding a narrow band around the epidemic threshold (details in \cref{app:sir-simulation-details}). Each simulation vector is a normalised infectious trajectory $I(t)/N$ on a uniform grid of $30$ time-points over $[0, 50]$ with additive Gaussian noise of standard deviation $0.05$. To distil equations, we use 500 simulations in $(\beta,\gamma)$. To demonstrate the usefulness of the learned coordinates, we extend our validation inference problem to include the initial number of infected individuals, $I_0$. The full parameter vector at inference is then $\theta = (\beta, \gamma, I_0/10)$. 

\textbf{Results.}
{We recover the symbolic coordinates
\[
        \eta_1 = \frac{0.64\beta}{1.41\gamma - 0.72} + 0.54,
        \quad
        \eta_2 = 0.55\beta + 0.67
\]
which closely track the analytically expected $R_0 \propto \beta/\gamma$
direction in the first component, flattening the learned Fisher matrices shown in 
\cref{fig:sir_results}. Downstream neural posterior estimation using amortised $\eta$ coordinates provides sharper $R_0$ posterior reconstruction (measured by CRPS) and
attains the same validation log-probability as the $\theta$-space baseline
using up to five times fewer simulations in low-simulation settings, since $\bm{\eta}$ more closely tracks parameter fluctuations in the data (details in \cref{app:sir-simulation-details}).}

\paragraph{Gravitational-wave chirp mass.}
Here we apply our method to gravitational-wave (GW) data, with the goal of
inferring the component masses $\theta = (m_1, m_2)$ of a compact binary
black-hole system from full waveform data. We compare two simulation regimes with increasing complexity. In both cases, we benchmark discovered coordinates against the conventional chirp-mass / mass-ratio combination
$\mathcal{M}_c = (m_1 m_2)^{3/5}/(m_1+m_2)^{1/5}$, $q = m_2/m_1$ used in GW astronomy. 

\textbf{Run 1.} We use TaylorF2 frequency-domain inspiral waveforms
\citep{Buonanno2009,Blanchet2014} at 2PN order (non-spinning), whitened by
the Advanced LIGO design sensitivity \citep{aLIGO2015} and PCA-compressed
to $40$ dimensions with additive unit Gaussian noise in the whitened
domain, at a luminosity distance of $200$\,Mpc.
The Fisher matrix varies non-trivially over the uniform prior in
$m_1 \geq m_2 \in [5, 50]\,M_\odot$. Heavy systems terminate at lower
$f_{\rm ISCO} \propto M^{-1}$ and accumulate fewer in-band cycles, so no
closed-form flat coordinates in $m_1,m_2$ exist. 

\textbf{Results.}
For masses rescaled to the $[0.1, 1.0]$ range, the distillery returns $\eta_1 = 0.2\,m_1,\
\eta_2 = 5.5\,\exp\!\left(-0.08\,(m_1+m_2)^2\right)$,
which is invertible in the prior studied. The second coordinate captures the inverse
of the chirp-mass-like sensitivity along the heavy-mass direction, and together the learned coordinates are a bijection of $(\mathcal{M}_c, q)$. 

{\textbf{Run 2. Beyond the chirp-mass approximation.}
TaylorF2 only captures the inspiral phase of the binary signal; for
moderately heavy systems the merger and ringdown carry most of the
in-band signal-to-noise and the leading-order
$\mathcal{M}_c \propto (m_1 m_2)^{3/5}/(m_1+m_2)^{1/5}$ scaling no
longer dominates the waveform. We repeat the experiment with the phenomenological inspiral-merger-ringdown
waveform \textsc{IMRPhenomD} \citep{Khan2016IMRPhenomD,Husa2016IMRPhenomD},
keeping the prior and PCA pipeline of the TaylorF2 case
unchanged. 

\textbf{Results.} The distillery now returns coordinates that do not contain the leading-order chirp-mass factor explicitly: the symbolic outputs are
\[
\eta_1 = 0.34\,(m_1 + m_2)^2 - 1.9, \quad
\eta_2 = -0.1\,m_1 - 0.3\,m_2 + 0.10\,(m_1+m_2)^2,
\]
i.e.\ a monotonic function of total mass $g(M)$ and the asymmetric combination
$\Delta m = m_1 - m_2$, with no chirp-mass-like factor anywhere. The discovered
coordinates are dramatically flatter than $(\mathcal{M}_c, q)$ for the
\textsc{IMRPhenomD} case, and the downstream posteriors are noticeably
rounder and better TARP-calibrated \citep{lemos2023samplingbasedaccuracytestingposterior} than the chirp-mass baseline
(\cref{fig:gw_imr}; details in \cref{app:gw-imr-details}).} This result indicates that the distillery has identified a more informative, data-driven combination of parameters.
\paragraph{Weak gravitational lensing.}
Weak gravitational lensing (WL) alters the trajectories of photons as they
pass through massive structures of visible and dark matter. The two-point
correlation function of this observable is sensitive to two degenerate
parameters, $(\Omega_{\rm m}, \sigma_8)$, which control the Universe's
matter content and dark-matter clustering, respectively. The nonlinear
parameter $S_8 = \sigma_8(\Omega_{\rm m}/0.3)^{0.5}$ is conventionally used
to summarise the
leading amplitude degeneracy to
changes in the two-point function
\citep{Jain1997,Bernardeau1997,Kilbinger2015}. We apply our method to 2-point statistics from mock WL convergence images, computed from expensive dark matter simulations (see App. \ref{app:WL}), varied over wide priors in $(\Omega_m, \sigma_8)$ and initial conditions. 

\textbf{Results.} We recover flattened coordinates $\eta_1 = \Omega_m\sigma_8 - 0.9\,\Omega_m^{0.144}$ and $\eta_2 = 0.5(\Omega_m - 1.0)$. These coordinates are similar to the standard $(\Omega_m, S_8)$ parameterization. In particular, $\eta_1$ captures the dominant weak-lensing amplitude degeneracy: near typical values $\Omega_m\simeq0.3$ and $\sigma_8\simeq0.8$, contours of constant $\eta_1$ have local slope $d\ln\sigma_8/d\ln\Omega_m \simeq -0.55$, close to the constant-$S_8$ slope of $-0.5$.

\begin{figure}
    \centering
    \includegraphics[width=\linewidth]{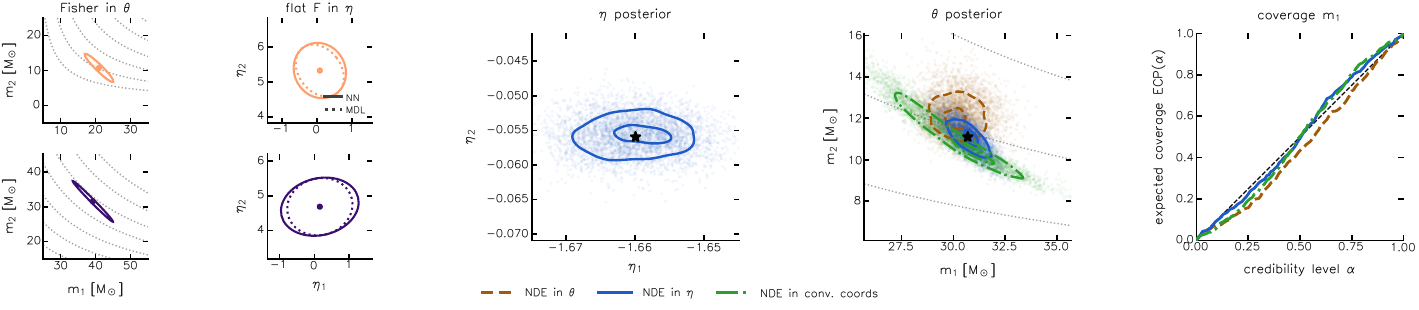}
    \caption{{Main GW result (\textsc{IMRPhenomD}). \textit{Left}: Example Fisher matrices in $(m_1, m_2)$ and in the learned $\eta$
    coordinates. \textit{Middle}: posterior contours from neural
    posterior estimators trained in $\theta = (m_1, m_2)$, in the
    $\eta$ basis, and in the conventional $(\mathcal{M}_c, q)$ basis. The
    $\eta$-trained posterior is round and centred on the truth, while
    both $\theta$-space and $(\mathcal{M}_c, q)$ posteriors stretch into
    bananas along the $\mathcal{M}_c$ direction (grey contours).
    \textit{Right}: TARP joint coverage; $\eta$ is closest to the diagonal across the full credibility range, indicating the learned coordinates are well-calibrated and more informative.}
    }
    \label{fig:gw_imr}
\end{figure}

\section{Discussion and Limitations}
\label{sec:discussion}
Across all experiments, the distillery captured information geometry and produced amortised, symbolic coordinates that make downstream inference more simulation-efficient and better calibrated in low- and noisy-data regimes, and provide physical insight about how observables are controlled by model parameters of interest. 
\paragraph{Limitations.}
{Our method has four main limitations. (i) Genuinely \emph{multimodal}
posteriors cannot be flattened by an invertible reparametrisation, since the
Fisher information is a local curvature object that is blind to global
modes. (ii) Very high-dimensional parameter spaces ($n_\theta \gtrsim 100$)
require large networks for stable Fisher estimation, and symbolic
regression scales poorly in the number of input variables. (iii) Exact
degeneracies (where $\mat{F}$ is singular) cannot be resolved by an
invertible map and require a prior reduction step \citep{Transtrum_2014}, although this can be mitigated in most cases with observational noise. 
(iv) The Fishnet stage is only as accurate as the
underlying Gaussian-posterior assumption (\cref{app:min_post_cov}); for
strongly non-Gaussian likelihoods, the recovered metric is the inverse of
the posterior covariance rather than the Fisher.} 


We introduced the \textit{degeneracy distillery}, a method for automatically
discovering and resolving parameter degeneracies directly from
parameter-data pairs. By estimating the Fisher information geometry with
neural networks, learning coordinate transformations that flatten this
geometry, and extracting symbolic representations of these transformations,
our approach produces interpretable reparametrisations that align with the
true directions of sensitivity in a model. This yields coordinates in which
the Fisher information is approximately isotropic, improving conditioning
for optimisation and sampling, while simultaneously exposing the
combinations of parameters that meaningfully control observables. Across synthetic and real-world examples, we demonstrated that the distillery
can recover known degeneracies as well as uncover previously unrecognised
structure, providing both computational and scientific benefits. In
particular, the resulting symbolic transformations offer a compact and
human-interpretable description of physical parameter dependencies. By focusing computation on
informative directions, the method can reduce the cost of
simulation-based inference and improve the robustness of downstream tasks. The degeneracy distillery provides a principled and practical tool for
uncovering and resolving degeneracies, offering both improved numerical
performance and deeper insight into the structure of parametrised models.
\section*{Broader Impact}
{This method is a tool for understanding and accelerating scientific
inference. The most direct impact is on simulation-based inference for
problems where simulation cost is the dominant bottleneck. By focusing learning from simulations along informative directions, the
distillery can reduce the energy and compute cost of large inference
campaigns and uncertainty quantification. The amortised, symbolic outputs additionally mitigates the opacity of black-box neural simulation-based inference.
We do not foresee specific negative societal impacts beyond those generic to
machine-learning tools deployed in scientific domains, although the
epidemiology application could be misapplied if used for policy decisions
without expert calibration.}

\section*{Data Availability}
The data underlying this paper will be shared on reasonable request to
the corresponding author. An open-source version of the code is available at \url{https://github.com/tlmakinen/degeneracy_distillery}.

\section*{Acknowledgements}
TLM acknowledges support from the InfoSys Cambridge AI Lab and Imperial
College London's President's Scholarship, as well as fruitful discussions with Will Handley, Alan Heavens, Natalia Porqueres, Roberto Trotta, Laurence Perreault-Levasseur, and François Lanusse.
DJB was supported by the Simons Collaboration on ``Learning the Universe'',
and acknowledges that support was provided by Schmidt Sciences, LLC. NJ is supported by the ERC-selected UKRI Frontier Research Grant EP/Y03015X/1.

\bibliographystyle{plainnat}
\bibliography{main}

\appendix

{The appendices are organised as follows.
\cref{app:min_post_cov} proves that the Fishnet objective recovers the
posterior mean and covariance, and motivates its use as a Fisher estimator.
\cref{app:unique sol} describes the alignment scheme used to fix the
constant-offset and orthogonal symmetries of the flattening loss when
ensembling.
\cref{app:architecture} details the Fishnet, flattener, and symbolic-
regression architectures and training hyperparameters.
\cref{app:gauss_normal_coords} constructs geodesic normal coordinates for
the Gaussian validation example used in \cref{sec:synthetic_validation}.
\cref{app:experiments} collects per-experiment simulation, prior, and
downstream-inference details.}

\section{Forward-KL Minimisation Yields the Posterior Moments}
\label{app:min_post_cov}

This appendix establishes that the Fishnet objective \eqref{eq:lossfn},
when treated as a variational loss over a Gaussian family, is minimised by
the posterior mean and posterior covariance of $\theta$ given $X$. The
identification with the Fisher information then follows in the standard
Laplace / Bernstein--von~Mises regime.

\paragraph{Setup.}
Let $X$ be observed data and $\theta\in\mathbb{R}^d$ parameters, with joint
density $p(X,\theta)$ and posterior $p(\theta\mid X)$. Throughout we
assume the posterior has finite first and second conditional moments
\begin{equation}\label{eq:post_moments}
m(X) := \mathbb{E}_{p(\theta\mid X)}[\theta], \qquad
S(X) := \mathrm{Cov}_{p(\theta\mid X)}[\theta]\in\mathbb{S}_{++}^{d},
\end{equation}
where $\mathbb{S}_{++}^{d}$ denotes the cone of $d\times d$ symmetric
positive-definite matrices. The Fishnet model parametrises a Gaussian
variational family
\begin{equation}\label{eq:gauss_family}
q_\varphi(\theta\mid X) = \mathcal{N}\!\bigl(\theta;\,\mu_\varphi(X),\,\Sigma_\varphi(X)\bigr),
\quad \Sigma_\varphi(X)\in\mathbb{S}_{++}^{d},
\end{equation}
through the network outputs $\mu_\varphi(X) := \hat{\theta}(X)$ and
$\Sigma_\varphi(X) := \mathbf{F}_\varphi(X)^{-1}$, where
$\mathbf{F}_\varphi$ is the log-Cholesky-parametrised matrix returned by
the network. Writing $\mathcal{L}$ for the loss \eqref{eq:lossfn}, the
per-sample loss is
\begin{equation}\label{eq:fishnet_loss_local}
\ell(\varphi; X, \theta) =
\tfrac{1}{2}\bigl(\theta - \mu_\varphi(X)\bigr)^{\!\top}\mathbf{F}_\varphi(X)\bigl(\theta - \mu_\varphi(X)\bigr)
- \tfrac{1}{2}\log\det\mathbf{F}_\varphi(X),
\end{equation}
and the population loss is $\mathbb{E}_{p(X,\theta)}[\ell(\varphi; X, \theta)]$.

\subsection{Main result}

\begin{proposition}[Forward-KL Gaussian moment matching]
\label{prop:moment-matching}
Fix $X$. Within the Gaussian family \eqref{eq:gauss_family}, the conditional
forward KL divergence
$D_{\mathrm{KL}}\!\bigl(p(\theta\mid X)\,\big\|\,q_\varphi(\theta\mid X)\bigr)$
admits a unique minimum over
$(\mu_\varphi,\,\mathbf{F}_\varphi)\in\mathbb{R}^{d}\times\mathbb{S}_{++}^{d}$
at
\begin{equation}\label{eq:moment_matching}
\mu_\varphi^{\star}(X) = m(X), \qquad
\Sigma_\varphi^{\star}(X) = \mathbf{F}_\varphi^{\star}(X)^{-1} = S(X),
\end{equation}
the posterior mean and posterior covariance of $\theta$ given $X$.
\end{proposition}

\begin{proof}
Direct expansion of the Gaussian density in \eqref{eq:gauss_family} gives,
for any $(\mu_\varphi,\mathbf{F}_\varphi)$,
\begin{equation}\label{eq:loss_KL_identity}
\mathbb{E}_{p(\theta\mid X)}\bigl[\ell(\varphi; X, \theta)\bigr]
= D_{\mathrm{KL}}\!\bigl(p(\theta\mid X)\,\big\|\,q_\varphi(\theta\mid X)\bigr)
+ H\!\bigl(p(\theta\mid X)\bigr) - \tfrac{d}{2}\log(2\pi),
\end{equation}
where $H(\cdot)$ denotes differential entropy. Both correction terms are
independent of $\varphi$, so minimising \eqref{eq:fishnet_loss_local} in
expectation and minimising the conditional KL are equivalent.

Decompose $\theta - \mu_\varphi = (\theta - m) + (m - \mu_\varphi)$ in the
quadratic form. The cross-term vanishes because $\mathbb{E}_p[\theta - m]=0$,
and the trace identity
$\mathbb{E}_p\!\left[(\theta - m)^{\!\top}\!A(\theta - m)\right]=\mathrm{tr}(AS)$,
valid for any symmetric $A$, reduces \eqref{eq:fishnet_loss_local} to
\begin{equation}\label{eq:reduced_loss}
\mathbb{E}_p\bigl[\ell(\varphi; X, \theta)\bigr]
= \underbrace{\tfrac{1}{2}(m - \mu_\varphi)^{\!\top}\mathbf{F}_\varphi (m - \mu_\varphi)}_{\text{(I)}\;\geq 0}
+ \underbrace{\tfrac{1}{2}\mathrm{tr}(\mathbf{F}_\varphi S) - \tfrac{1}{2}\log\det\mathbf{F}_\varphi}_{\text{(II)}}.
\end{equation}
Term~(II) does not depend on $\mu_\varphi$, and term~(I) is a non-negative
quadratic form that vanishes if and only if $\mu_\varphi = m$ (since
$\mathbf{F}_\varphi\succ 0$). For any $\mathbf{F}_\varphi\in\mathbb{S}_{++}^{d}$,
replacing $\mu_\varphi$ by $m$ therefore strictly decreases
\eqref{eq:reduced_loss} unless $\mu_\varphi = m$ already; any minimiser
must satisfy $\mu_\varphi^{\star}=m$.

Substituting $\mu_\varphi = m$ leaves the precision objective
\begin{equation}\label{eq:precision_obj}
g(\mathbf{F}) := \tfrac{1}{2}\mathrm{tr}(\mathbf{F} S) - \tfrac{1}{2}\log\det\mathbf{F},
\qquad \mathbf{F}\in\mathbb{S}_{++}^{d}.
\end{equation}
The map $\mathbf{F}\mapsto -\log\det\mathbf{F}$ is strictly convex on
$\mathbb{S}_{++}^{d}$ and $\mathbf{F}\mapsto\mathrm{tr}(\mathbf{F} S)$ is
linear, so $g$ is strictly convex on the (convex) cone $\mathbb{S}_{++}^{d}$.
Standard matrix calculus gives
$\nabla_{\mathbf{F}}\,g(\mathbf{F}) = \tfrac{1}{2}\bigl(S - \mathbf{F}^{-1}\bigr)$,
which vanishes uniquely at $\mathbf{F}^{\star}=S^{-1}$, hence
$\Sigma_\varphi^{\star}=S$. The pair
$(\mu_\varphi^{\star},\mathbf{F}_\varphi^{\star}) = (m, S^{-1})$ is the
unique global minimiser of \eqref{eq:reduced_loss}, and by
\eqref{eq:loss_KL_identity} also of the conditional KL.
\end{proof}

\begin{corollary}[Fishnets recover the posterior moments]
\label{cor:fishnets-recover}
Suppose the variational class
$\{(\mu_\varphi,\,\mathbf{F}_\varphi):\varphi\in\Phi\}$ is sufficiently
expressive to represent the conditional moment maps $X\mapsto m(X)$ and
$X\mapsto S(X)^{-1}$, and that the training distribution covers $p(X)$.
Then any global minimiser of the population Fishnet loss
$\mathbb{E}_{p(X,\theta)}[\ell(\varphi; X, \theta)]$ satisfies, for
$p(X)$-almost every $X$,
\begin{equation}\label{eq:fishnet_recovers_moments}
\hat{\theta}(X) = \mathbb{E}_{p(\theta\mid X)}[\theta], \qquad
\mathbf{F}_\varphi(X) = \mathrm{Cov}_{p(\theta\mid X)}[\theta]^{-1}.
\end{equation}
\end{corollary}

\begin{proof}
By the tower property,
$\mathbb{E}_{p(X,\theta)}[\ell] = \mathbb{E}_{p(X)}\!\bigl[\mathbb{E}_{p(\theta\mid X)}[\ell]\bigr]$.
Applying Proposition~\ref{prop:moment-matching} pointwise in $X$ and using
the identity \eqref{eq:loss_KL_identity}, the conditional minimum is
attained at \eqref{eq:moment_matching}; expressivity of the variational
class allows this minimum to be realised simultaneously for $p(X)$-almost
every $X$.
\end{proof}

\begin{remark}[Posterior covariance versus Fisher information]
\label{rem:fisher-vs-cov}
Corollary~\ref{cor:fishnets-recover} identifies $\mathbf{F}_\varphi(X)$
with the inverse \emph{posterior} covariance, not the likelihood Fisher
information
\begin{equation*}
\mathcal{I}(\theta) := -\mathbb{E}_{p(X\mid\theta)}\bigl[\nabla_\theta^{\,2}\log p(X\mid\theta)\bigr].
\end{equation*}
Under standard regularity, the local Laplace approximation gives
$p(\theta\mid X) \approx \mathcal{N}\!\bigl(\hat\theta_{\rm MAP},\,H(\hat\theta_{\rm MAP})^{-1}\bigr)$
with posterior curvature
$H(\theta) = \mathcal{J}(\theta;X) - \nabla_\theta^{\,2}\log p(\theta)$,
where $\mathcal{J}(\theta;X) := -\nabla_\theta^{\,2}\log p(X\mid\theta)$ is
the observed information. In the Bernstein--von~Mises regime of $n$ i.i.d.\
observations, $H(\theta_0)\to n\,\mathcal{I}(\theta_0)$ in probability and
the prior contribution becomes negligible, so $\mathbf{F}_\varphi(X)$
asymptotically recovers $n\,\mathcal{I}(\theta_0)$. Outside this regime
$\mathbf{F}_\varphi$ should be read as an inverse posterior covariance,
which is the relevant object for the flattening pipeline of
\cref{sec:method}: the metric \eqref{eq:metric_transform} transforms
identically under $\theta\to\eta$ whether one substitutes
$\mathcal{I}(\theta)$ or $\mathrm{Cov}_{p(\theta\mid X)}[\theta]^{-1}$.
\end{remark}

\section{Obtaining a unique solution in flattened coordinates}
\label{app:unique sol}

The loss function we minimise (\cref{eq:lossfn}) has two important symmetries:
\begin{enumerate}
    \item $L(\gen{\bm{\eta}})$ is invariant if $\gen{\eta}^a \to \gen{\eta}^a + c^a$, for some constant vector $c^a$ (this is trivial from the invariance of $F_{ab}$ to this).
    \item $L(\gen{\bm{\eta}})$ is invariant under global orthogonal transformations. To prove this, consider new coordinates $\{\tilde{\phi}^a\}$ such that $\gen{\eta}^a = R^{a}{}_b \tilde{\phi}^b$ where $\mat{R}\mat{R}^{\rm T} = \mat{1}$. The new Fisher is
    \begin{equation}
        \tilde{\mat{F}}(\phi) = \mat{R}^{\rm T}\tilde{\mat{F}}(\theta)\mat{R}. 
    \end{equation}
    Using the cyclic property of the trace, one sees that $L(\tilde{\phi}) = L(\gen{\eta})$.
\end{enumerate}
Optimising the Frobenius norm loss does not, therefore, yield a unique set of parameters, but one that is known only up to a constant offset and an orthogonal transformation.
This is particularly problematic when we train an ensemble of networks, since each member of the ensemble will have a different offset and orthogonal transformation relative to some fiducial network. 

To deal with this, we employ a two-step rotation scheme to align ensemble members $\{ \gen{\bm{\eta}} \}$. First, we choose a reference ensemble member and employ a Householder reflection \citep{householder_unitary_1958} to rotate its learned coordinates at a fiducial value $\gen{\eta}^* = \gen{\eta}(\theta^*)$ such that $\gen{\eta}^*$ points along the first eigenvalue of the variance-weighted Fisher matrix:
\begin{equation}
    \mat{F}^* / \delta\mat{F} v = \lambda v, \quad
    R^H \gen{\eta}^* = \lambda_1.
\end{equation}

The Householder reflection is given as
\begin{equation}
    \text{ref}_{\mathbf{n}}(\mathbf{u}) = \mathbf{u} - 2\frac{\mathbf{n} \mathbf{n}^\top \mathbf{u}}{\mathbf{n}^\top \mathbf{n}}
\end{equation}
Where $\mathbf{n}$ is the normal to the reflection plane and
$\mathbf{n}^\top \mathbf{n}$ is the squared norm $\|\mathbf{n}\|^2$. This can be cast as two successive reflections from $ \hat{\mathbf{x}} = \frac{\mathbf{x}}{\|\mathbf{x}\|}$, $\hat{\mathbf{y}} = \frac{\mathbf{y}}{\|\mathbf{y}\|}$, one across the hyperplane defined by the sum of the two vectors, and one across the target vector:
\begin{align}
    S &= I - 2\frac{(\hat{\mathbf{y}} + \hat{\mathbf{x}})(\hat{\mathbf{y}} + \hat{\mathbf{x}})^\top}{(\hat{\mathbf{y}} + \hat{\mathbf{x}})^\top (\hat{\mathbf{y}} + \hat{\mathbf{x}})} \\
    R &= S - 2\frac{\hat{\mathbf{y}} \hat{\mathbf{y}}^\top S}{\hat{\mathbf{y}}^\top \hat{\mathbf{y}}},
\end{align}
where $R$ satisfies
\begin{equation}
    R\hat{\mathbf{x}} = \hat{\mathbf{y}}.
\end{equation}

Once this reference point is established, we employ a Kabsch rotation \citep{kabsch_solution_1976} to re-align each ensemble member's coordinates. We wish to align each ensemble member's output centroid-subtracted coordinates $P = [\gen{\eta}^i_j - \langle \gen{\eta}^i \rangle]\in \mathbb{R}^{N \times n_p}$, where $N$ is the number of samples and $n_p$ is the number of parameters,
to the reference ensemble member's centroid-subtracted coordinates $P = [\gen{\eta}^{0}_j - \langle\gen{\eta}^0\rangle]\in \mathbb{R}^{N \times n_p}$. The Kabsch algorithm finds an optimal rotation $R^*$ which minimises
\begin{equation}
    \mathcal{L}_{\rm kabsch} = \frac{1}{2}\sum_{i=0}^N || Q - RP||^2.
\end{equation}
The optimal rotation can be found from the Singular Value Decomposition (SVD) of the point cloud covariance matrix, $C = P^T Q$:
\begin{equation}
    C = U\Sigma V^T, \quad
    R^* = VBU^T,
\end{equation}
where $d = {\rm sign}(\det VU^T)$ and $B = {\rm diag}(1, 1, d)$ fixes rotations to be right-handed.

{
\paragraph{Jacobian Procrustes alignment.}
Coordinate Kabsch matches ensemble members in the
$\gen{\bm{\eta}}$-space, but the loss \eqref{eq:flat_loss} cares only about
the Jacobian $\mat{\mathcal{J}} = \partial \gen{\bm{\eta}}/\partial \bm{\theta}$
through the metric pullback. As an option to users we also provide another alignment scheme to highlight nonlinear flattened components along directions of high variance in the network Jacobian. We perform an orthogonal Procrustes fit performed directly on the batch of Jacobians:
given source Jacobians $\{\mat{{J}}^{(i)}_b\}_{b=1}^B$ for ensemble
member $i$ and reference Jacobians $\{\mat{{J}}^{(0)}_b\}_{b=1}^B$,
the rotation acts on the $\gen{\bm{\eta}}$-axis (left index) of
$\mat{{J}}$, $\mat{{J}}'_b = \mat{R}\,\mat{{J}}^{(i)}_b$,
and the alignment problem
\begin{equation}\label{eq:jacobian_procrustes}
    \mat{R}^* \;=\; \arg\min_{\mat{R}\in O(n_p)}
    \sum_{b=1}^{B} w_b\,\big\|\mat{R}\,\mat{{J}}^{(i)}_b
    - \mat{{J}}^{(0)}_b\big\|_{\mathcal{F}}^{2}
\end{equation}
admits the closed-form solution $\mat{R}^* = \mat{U}\mat{V}^{\rm T}$, where
$\mat{U}\mat{S}\mat{V}^{\rm T}$ is the SVD of the cross-covariance
$\mat{M} = \sum_b w_b\,\mat{{J}}^{(0)}_b\bigl(\mat{{J}}^{(i)}_b\bigr)^{\!\top}$.
The sample weights $w_b$ optionally reflect the validation weighting from
\cref{eq:lossfn}. Compared with coordinate Kabsch, this scheme uses every
sample and unit-magnitude derivatives instead of noisy coordinates, so it is
markedly more stable on near-degenerate axes whose $\bm{\eta}$-displacements
are dominated by ensemble noise. We allow improper rotations
($\det \mat{R}^* = -1$) per member when this strictly reduces
\eqref{eq:jacobian_procrustes}: the flatness loss
$\big\|\mat{R}\bar{\mat{F}}^\prime\mat{R}^{\rm T} - \mat{1}\big\|_\mathcal{F}^2$
is invariant under arbitrary $\mat{R}\in O(n_p)$, so the only requirement is
that all members land in a common frame, which the Procrustes solution achieves.}

{
\paragraph{Member-independent canonical basis.}
After Procrustes alignment the ensemble shares a single orthogonal frame, but
the choice of frame is still arbitrary. We fix it once on the reference member
by composing two operations. First, a \emph{nonlinearity-separating rotation}
takes the SVD of the sample-centred Jacobian stack
$\Delta\mat{{J}}_b = \mat{{J}}^{(0)}_b - \bar{\mat{{J}}}^{(0)}$
(with optional column rescaling by the prior widths
$\Delta\theta_a = \theta_a^{\max}-\theta_a^{\min}$ to put all parameters on a
dimensionless footing): reshaping
$\Delta\mat{{J}}\in\mathbb{R}^{B\times n_p\times n_p}$ to
$\Omega\in\mathbb{R}^{n_p\times Bn_p}$ and writing
$\Omega = \mat{U}{\Sigma}\mat{V}^{\rm T}$, the rotation
$\mat{R}_{\rm nl} = \mat{U}^{\rm T}$ orders the $\gen{\bm{\eta}}$-axes by
decreasing nonlinearity ``energy'' (singular values of $\Omega$): axes whose
Jacobian rows are sample-independent (i.e.\ $\gen{\eta}_a$ is a linear function
of $\bm{\theta}$) collect into the trailing rows. Second, a sign and (optionally)
permutation fix-up $\mat{R}_{\rm canon}$ uses the eigenstructure of the mean
prior-normalised $\bm{\theta}$-Fisher
$\bar{F}_{ab}/(\Delta\theta_a\,\Delta\theta_b)$ to flip per-axis signs so that
$\langle (\mat{R}_{\rm nl}\,\mat{{J}})_a, v_a\rangle \geq 0$ for the
corresponding Fisher eigenvector $v_a$. Composing $\mat{R}_{\rm basis}=
\mat{R}_{\rm canon}\mat{R}_{\rm nl}$ once on the reference and applying
$\mat{R}_{\rm total}^{(i)} = \mat{R}_{\rm basis}\,\mat{R}^{*\,(i)}$ to every
member guarantees the same axis ordering and orientation across the ensemble,
which keeps the dispersion of $\gen{\bm{\eta}}^{(i)}$ tight on the
near-degenerate (least-nonlinear) axes.}

\section{Architecture details}
\label{app:architecture}

\subsection{Fisher Information Neural Networks}

We train an ensemble of neural networks (Fishnets) to predict Fisher information matrices $F(\theta)$ directly from simulation data $\textbf{x}$.
Each network comprises a residual MLP followed by a Fisher prediction layer.
The Fisher matrices are parameterised via log-Cholesky decomposition to ensure positive-definiteness:
\begin{equation}
F = LL^\top, \quad L_{ii} = \mathrm{softplus}(\ell_i) + 10^{-4},
\end{equation}
where $\ell_i$ are learned diagonal parameters and off-diagonal elements are unconstrained.

\subsubsection{SoftSquelch activation}

The Fisher prediction branch employs a learned sparse activation function:
\begin{equation}
\mathrm{SoftSquelch}(x) = x \cdot \sigma\bigl(s(|x| - t)\bigr),
\end{equation}
where $\sigma$ is the sigmoid function, and threshold $t$ and sharpness $s$ are hyperparameters.
This suppresses small activations whilst preserving large ones, promoting sparsity.

\subsubsection{Ensemble hyperparameter priors}
Ensemble diversity is achieved through randomised architectural choices, which contribute to pushing the sampling distribution of learned Fisher matrix entries and downstream flattened coordinates towards Gaussian. Hidden layer widths are sampled uniformly from a user-specified range (default $[10, 300)$ neurons), and activation functions are drawn from a discrete distribution:
\begin{align}
p(\mathrm{activation}) = 
\begin{cases}
3/23 & \text{ReLU} \\
4/23 & \text{Leaky ReLU} \\
3/23 & \text{Swish} \\
2/23 & \text{Mish} \\
3/23 & \text{smooth\_leaky} \\
8/23 & \text{GELU}
\end{cases}
\end{align}
where Mish is defined as $\mathrm{Mish}(x) = x \tanh(\mathrm{softplus}(x))$, and smooth\_leaky is defined as:
\begin{equation}
\mathrm{smooth\_leaky}(x) = \alpha x + (1 - \alpha) x \log\sigma(x), \quad \alpha = 0.2,
\end{equation}
with $\log\sigma(x) = -\log(1 + e^{-x})$ ensuring no discontinuities in derivatives. The SoftSquelch hyperparameters are sampled from independent Gaussians:
\begin{align}
s &\sim \mathcal{N}(5.0, 0.7^2), \\
t &\sim \mathcal{N}(1.0, 0.7^2),
\end{align}
where $s$ controls the sharpness of the transition and $t$ sets the suppression threshold.

\subsubsection{Training objective}

Networks minimise the Kullback--Leibler divergence between predicted and true parameter distributions:
\begin{equation}
\mathcal{L}_{\mathrm{KL}} = \frac{1}{2}\mathbb{E}\Bigl[\mathrm{tr}\bigl(F(\hat{\theta} - \theta)(\hat{\theta} - \theta)^\top\bigr) - \log\det F\Bigr],
\end{equation}
where $\hat{\theta}$ denotes the predicted parameter estimate.

\subsection{Geometric Flattening}

We learn a coordinate transformation $\eta = f(\theta; w)$ that renders the Fisher metric approximately Euclidean.
The network architecture is a residual MLP with inverse whitening:
\begin{equation}
\eta = W^{-1} \cdot \mathrm{MLP}(\theta),
\end{equation}
where $W^{-1} = F_{\mathrm{mean}}^{1/2}$ is a fixed inverse whitening matrix computed from the global mean Fisher.

\subsubsection{Flattening loss}

The transformation is learned by minimising deviations of the pullback metric from the identity:
\begin{equation}
\mathcal{L}_{\mathrm{flat}} = \Bigl\| J^{-\top}FJ^{-1} - I \Bigr\|_F + \Bigl\| (J^{-\top}FJ^{-1})^{-1} - I \Bigr\|_F,
\end{equation}
where $J = \partial\eta/\partial\theta$ is the Jacobian. To speed up progress on hard points, each $\ell$ is rescaled by an adaptive
factor
\[
  r(\mathcal{L}) \;=\; \frac{\lambda\,\mathcal{L}}{\ell + \mathrm{e}^{-\alpha\mathcal{L}}},
  \qquad
  \alpha = \alpha(\lambda,\varepsilon),
\]
and the training objective is the batch mean of $\log\bigl(r(\mathcal{L})\,\mathcal{L}\bigr)$.  Larger
$\lambda$ makes the map $r(\cdot)$ more aggressive about emphasizing
large-residual samples relative to easy ones.

\subsubsection{Whitening}

The inverse whitening layer implicitly transforms the training objective.
With $\eta = W^{-1}\eta_{\mathrm{raw}}$ and $J = W^{-1}J_{\mathrm{raw}}$, the loss becomes:
\begin{equation}
\mathcal{L}_{\mathrm{flat}} = \Bigl\| J_{\mathrm{raw}}^{-\top}(WFW)J_{\mathrm{raw}}^{-1} - I \Bigr\|_F,
\end{equation}
effectively training on pre-whitened Fishers without modifying the data.

\subsection{Symbolic Coordinate Discovery}

We apply symbolic regression (\textsc{operon}, \citep{Burlacu_2020}) to the flattened coordinates $\eta$ to discover interpretable expressions.
Each component $\eta_i(\theta)$ is fitted independently using genetic programming with objectives balancing accuracy ($R^2$) and complexity (equation length).
The Pareto frontier yields candidate expressions.

{\paragraph{SR input augmentation.} The trained flattening ensemble defines a deterministic map $\bar{\eta}(\theta)$ with attached uncertainty $\delta\eta(\theta)$ that is independent of the simulations originally used to fit the Fisher--flattener stack. Rather than feeding only the (small) Fishnet validation set into SR, we draw a fresh dense sample from the prior support and evaluate $(\bar{\eta}, \delta\eta)$ on those points; the resulting triples are used as the Gaussian-loss SR targets. We find this materially improves SR convergence in every experiment we have tried (cleaner Pareto fronts, more reliable recovery of the analytic coordinates on Rosenbrock) and is essentially free since it requires no additional simulator calls; see \cref{app:experiments} for the per-experiment sample counts.}

\subsubsection{Model selection}

We rank expressions by multiple criteria:
\begin{itemize}
\item \textbf{Minimum description length (MDL):} Bayesian model comparison penalising complexity \citep{Bartlett_2023_priors,Bartlett_2024_ESR}.
\item \textbf{Frobenius loss:} Flattening quality $\| J_{\mathrm{symb}}^{-\top}FJ_{\mathrm{symb}}^{-1} - I \|_F$ on held-out data.
\item \textbf{$R^2$ score:} Prediction accuracy on test set.
\end{itemize}
The symbolic Jacobian $J_{\mathrm{symb}}$ is computed via automatic differentiation with Sympy \citep{sympy}.
Final coordinates are selected to minimise either the MDL or the Frobenius loss whilst maintaining interpretability.

\subsection{Postprocessing Symbolic Expressions}
\label{app:post_process}

Let $\ell_{\mathrm{flat}}(\theta)=\frac{1}{B}\sum_{b=1}^{B}\bigl\|\widetilde{F}_b(\theta)-I\bigr\|_F$
be the mean (over samples $b$) Frobenius distance of the flattened Fisher
$\widetilde{F}_b$ from the identity, for linear coefficients $\theta$ in the
symbolic coordinates.  Fix $\ell_{\mathrm{ref}}=\ell_{\mathrm{flat}}(\theta_0)$
at the current parameterization.

\emph{Importance (importance sampling / local sensitivity):} for each scalar
coefficient $c$ with $c\neq 0$, use a one-sided finite difference step
$\varepsilon$ and set
\[
  w_c \;=\; \bigl|c\bigr|\;
  \left|\frac{\ell_{\mathrm{flat}}(\theta_0+\varepsilon e_c)-\ell_{\mathrm{ref}}}{\varepsilon}\right|,
\]
where $e_c$ perturbs only that coefficient.

\emph{Pruning:} sort coefficients by increasing $w_c$.  Greedily try setting
coefficients to zero (optionally in batches); accept a proposed sparsification
$\theta'$ only if the \emph{relative} flattening increase stays below a
tolerance $\tau$,
\begin{equation}
    \frac{\ell_{\mathrm{flat}}(\theta')-\ell_{\mathrm{ref}}}{\ell_{\mathrm{ref}}} \;<\; \tau .
\end{equation}

\subsection{Invertible network mapping: Real NVP}

We optionally parameterize the reparameterization $\eta(\theta)$ with a Real NVP normalizing flow~\cite{dinh2017density}, implemented in JAX/Flax as a stack of affine coupling layers~\cite{dinh2017density,papamakarios2021normalizing}. Each layer fixes half the coordinates (alternating even/odd dimensions across layers) and applies an elementwise affine map to the remaining coordinates,
\[
  y_i = x_i \odot \exp\bigl(s(x_{\mathrm{fix}})\bigr) + t(x_{\mathrm{fix}}),
\]
where $s$ and $t$ are MLPs of one hidden layer with width \texttt{hidden\_dims}; the $s$-head uses a final $\tanh$ for bounded logits, and scales use $\exp$ by default so Jacobians stay tractable. The forward pass accumulates $\log|\det J|$ over layers; the inverse applies the couplings in reverse order. For training, $\theta$ is min--max scaled to $[0,1]$, shifted by one (so inputs lie in $[1,2]$), passed through the flow, and optionally left-multiplied by a fixed inverse-whitening matrix $\mathbf{W}^{-1}$ on $\eta_{\mathrm{raw}}$ (\texttt{WhitenedRealNVP}); the inverse reverses these steps in order.

{
\subsection{Forward--backward MLP}
\label{app:fwd_bwd_mlp}
As a lightweight alternative to the RealNVP flow, we pair the forward MLP $\eta = f_\phi(\theta)$ with a separate \emph{reverse-path} MLP $\hat\theta = g_\psi(\eta)$ that approximates the inverse, and add a mean-square cycle-consistency penalty $\|\theta - g_\psi(f_\phi(\theta))\|^{2}$ to the flattening loss with weight $\lambda_{\rm inv}$, so that the total objective is $\mathcal{L}_{\rm flat} + \lambda_{\rm inv}\,\mathbb{E}_\theta\|\theta - g_\psi(f_\phi(\theta))\|^{2}$; this softly enforces invertibility on the prior support without constraining the architecture as a normalising flow.}

\section{Geodesic normal coordinates for Gaussian example}
\label{app:gauss_normal_coords}

\subsection{Motivation}

The aim of this paper is to construct coordinates which make the Fisher information as flat as possible.
In \cref{sec:synthetic_validation}, we considered the example of a Gaussian, parametrised by mean $\mu$ and standard deviation $\sigma$. 
We will write these as coordinates $y^a = (\mu, \sigma)$. The Fisher information will be our metric, which is $F_{ab} = \diag(1/\sigma^2, 2/\sigma^2)$,
where we have ignored the
multiplicative factor of $n_{\rm d}$ for simplicity.

In this appendix, we wish to find an analytic coordinate transformation to compare against the one obtained by our method. As we will show, there does not exist a coordinate transformation which can make the Fisher equal to the identity everywhere, and thus we construct geodesic normal coordinates to find the coordinates which make the metric appear as flat as possible in the vicinity of some point.

Suppose we were at some special point $p$ with coordinates $y_\star^a = (\mu_\star, \sigma_\star)$ and wanted to make the metric flat in the vicinity of that point. 
For simplicity, let us begin by rescaling $\mu$ and $\sigma$ by functions of $\sigma_\star$ to make the metric equal to the identity at $p$. We define
\begin{equation}
    u \equiv \frac{\mu}{\sigma_\star},
    \quad
    v \equiv \sqrt{2} \frac{\sigma}{\sigma_\star},
\end{equation}
so that the corresponding line element for a metric $g_{ab} = F_{ab}$ becomes
\begin{equation}
    {\rm d} s^2 
    = \left( \frac{\sigma_\star}{\sigma} \right)^2 \left( {\rm d}u^2 + {\rm d} v^2 \right)
    = \frac{2}{v^2} \left( {\rm d}u^2 + {\rm d} v^2 \right),
\end{equation}
i.e. the metric in these coordinates is
\begin{equation}
    g_{ab} = \frac{2}{v^2} \delta_{ab}.
\end{equation}
We will use $x^a = (u, v)$.
The coordinates of the point $p$ are then
\begin{equation}
    u_\star = \frac{\mu_\star}{\sigma_\star}, \quad
    v_\star = \sqrt{2}.
\end{equation}

Although the metric is Euclidean at $p$ in these coordinates, we would see that as we moved away from $p$, the leading order terms for the metric would be linear in our coordinates.
A better set of coordinates to use is geodesic normal coordinates, since these give quadratic corrections to the metric at leading order:
\begin{equation}
    \label{eq:geodesic_metric_expansion}
    \tilde{g}_{ab} = \delta_{ab} - \frac{1}{3} R_{a c b d} (x_\star) \left( x^c - x_\star^c \right)\left( x^d - x_\star^d \right) + \mathcal{O} \left( \left| x - x_\star \right|^3\right),
\end{equation}
where $R_{abcd}$ is the Riemann tensor.
The aim of this appendix is to construct these coordinates for a Gaussian. Of course, the flatness is only valid for some region of space. This does, however, represent a good benchmark to see whether our ML method can provide something better over all space than this construction.

\subsection{Series expansion for geodesic normal coordinates}

Before explicitly constructing geodesic normal coordinates, let us first obtain a series expansion for the metric in these coordinates, so that we have something to verify our result against.

Let us first construct the Christoffel symbols for this metric, given by
\begin{equation}
    \Gamma^a_{bc} = \frac{1}{2} g^{a d} \left( \partial_c g_{d b} + \partial_b g_{d c} - \partial_d g_{bc} \right).
\end{equation}
For our metric, these are
\begin{equation}
    \Gamma^0_{ab} = 
    \begin{pmatrix}
        0 & - \frac{1}{v}\\
        - \frac{1}{v} & 0
    \end{pmatrix},
    \quad
    \Gamma^1_{ab} = 
    \begin{pmatrix}
        \frac{1}{v} & 0 \\
        0 & - \frac{1}{v}
    \end{pmatrix}.
\end{equation}
Now we must compute the Riemann tensor
\begin{equation}
    R^{d}_{cab} = \partial_a \Gamma^d_{bc} - \partial_b \Gamma^d_{ac} + \Gamma^d_{ae}\Gamma^e_{bc} - \Gamma^d_{be}\Gamma^e_{ac}.
\end{equation}
Since we are in two dimensions, this has only one independent component, and is thus fully determined by the Ricci scalar, $R$. For our metric, we find that
\begin{equation}
    \label{eq:gauss_ricci_scalar}
    R = -1,
\end{equation}
which yields the non-zero components of the Riemann tensor
\begin{equation}
    \label{eq:gauss_riemann_tensor}
    R_{0110} = R_{1001} = - R_{1010} = - R_{0101} =  \frac{1}{v^2}.
\end{equation}
Since the Riemann tensor is non-zero, we cannot find global coordinates which flatten the metric.
Instead, we are now in a position to write down the metric expansion in geodesic normal coordinates using \cref{eq:geodesic_metric_expansion}.
\begin{equation}
    \label{eq:gaussian_metric_expansion}
        \tilde{g}_{ab} \approx
    \mat{1} + \frac{1}{3 v_\star^2}
    \begin{pmatrix}
        \left( v - v_\star \right)^2 & - \left( u - u_\star \right) \left( v - v_\star \right) \\
        - \left( u - u_\star \right) \left( v - v_\star \right) & \left( u - u_\star \right)^2
    \end{pmatrix}
    + \mathcal{O} \left( \left| x - x_\star \right|^3\right).
\end{equation}

\subsection{Geodesic Equations}

To construct geodesic normal coordinates, we first need to compute the geodesics on this manifold. Let us use $t$ as our affine parameter, and use a dot to denote a derivative with respect to $t$. The geodesic equation is defined to be
\begin{equation}
    \frac{{\rm d}^2 x^a}{{\rm d} t^2} + \Gamma^a_{bc} \frac{{\rm d} x^b}{{\rm d}t} \frac{{\rm d} x^c}{{\rm d}t} = 0,
\end{equation}
hence our equations are 
\begin{equation}
    \ddot{u} - \frac{2}{v} \dot{u} \dot{v} = 0, \qquad
    \ddot{v} + \frac{\dot{u}^2}{v} - \frac{\dot{v}^2}{v} = 0.
\end{equation}

The equation for $u$ is easy to integrate once. If we multiply by $1 / v^2$, we note that this can be written as
\begin{equation}
    \frac{{\rm d}}{{\rm d}t} \left( \frac{\dot{u}}{v^2} \right) = 0 
    \quad \implies \quad
    \dot{u} = A v^2,
\end{equation}
where $A$ is a constant. Similarly, we note that we can write the equation for $v$ as 
\begin{equation}
    v \frac{{\rm d}}{{\rm d}t} \left( \frac{\dot{v}}{v} \right) + \frac{\dot{u}^2}{v}  = 
    v \frac{{\rm d}^2 \log v}{{\rm d}t^2} +\frac{\dot{u}^2}{v} = 0.
\end{equation}
We can then substitute in our result for $\dot{u}$ to obtain an equation for $v$ only. This becomes easier to solve if we define $w = \log v$, since this is
\begin{equation}
    \ddot{w} + A^2 e^{2 w} = 0,
\end{equation}
which has the solution
\begin{equation}
    \dot{w} = \pm \left( B - A^2 e^{2w} \right)^{\frac{1}{2}},
\end{equation}
where $B$ is a constant. This can then be integrated to obtain
\begin{equation}
    \pm \left( t - t_0 \right) = 
    \frac{-1}{\sqrt{B}} \tanh^{-1} \left( \sqrt{1 - \frac{A^2}{B} e^{2w}} \right),
\end{equation}
where $t_0$ is an integration constant.
Rearranging (squaring will remove the $\pm$), one finds
\begin{equation}
    v(t) = 
    \frac{\sqrt{B}}{|A|} \sech \left( \sqrt{B} \left( t - t_0 \right) \right),
\end{equation}
where we take the positive square root, since we will always have $v > 0$ for a Gaussian.
This can then be used to obtain an equation for $\dot{u}$
\begin{equation}
    \dot{u} = A v^2 = \frac{B}{A} \sech^2 \left( \sqrt{B} \left( t - t_0 \right) \right),
\end{equation}
which gives
\begin{equation}
    u(t) = \frac{\sqrt{B}}{A} \tanh \left( \sqrt{B} \left( t - t_0 \right) \right)
 + C,
\end{equation}
where $C$ is again a constant.

We will redefine our constants such that $\beta = \sqrt{B}$, and we will assume that the geodesic goes through the point $(u_\star, v_\star)$ at $t=0$. This gives us
\begin{align}
    u (t) &= u_\star + v_\star \cosh \left( \beta t_0 \right) \left[ \tanh \left( \beta \left( t - t_0 \right) \right) + \tanh(\beta t_0) \right], \\
    v (t) &= v_\star \cosh(\beta t_0) \sech \left( \beta \left( t - t_0 \right) \right) .
\end{align}
It is useful to rewrite this with the identity $\tanh x + \tanh y = \sech x \sech y \sinh (x + y)$, to give
\begin{align}
    u (t) &= u_\star + v_\star \sinh \left( \beta t_0 \right) \sech \left( \beta \left( t - t_0 \right) \right), \\
    v (t) &= v_\star \cosh(\beta t_0) \sech \left( \beta \left( t - t_0 \right) \right) .
\end{align}

\subsection{Geodesic Normal Coordinates}

Now we have found the geodesics which pass through the point $p$, we are in a position to construct geodesic normal coordinates about this point. 
These are constructed as follows. Suppose one has a vector $V^a$ which lives in the tangent space at $p$, then one can parallel transport this to the point $X$, which one defines as the point on the geodesic with affine parameter $t=1$ (the exponential map). In this case (provided the metric is Euclidean at $p$) the geodesic normal coordinates of $X$ are $V^a$.

Hence, to construct the geodesic normal coordinates, we need to know the tangent vector of our geodesic at the point $p$. If we define $V^a = (\xi, \eta)$, then these are
\begin{equation}
    \xi \equiv \dot{u}(0) = \beta v_\star \sech \left( \beta t_0 \right),
    \quad
    \eta \equiv \dot{v}(0) = \beta v_\star \tanh \left( \beta t_0 \right).
\end{equation}
We can therefore relate the remaining two free constants ($\beta$ and $t_0$) in our geodesic equation to the geodesic normal coordinates $\xi$ and $\eta$. 

It is useful to use the identity $\cosh(x - y) = \cosh x \cosh y - \sinh x \sinh y$ to split up the $\sech(\beta(t-t_0))$ term into terms which are functions of only $\beta t$ or $\beta t_0$. Doing this, and recalling the definition that $(u, v)$ is at the point $t=1$, we obtain the following expressions to relate $(u, v)$ to $(\xi, \eta)$
\begin{equation}
    u = u_\star + v_\star \frac{\xi \sinh \beta}{\beta v_\star \cosh \beta - \eta \sinh \beta}, \qquad
    v = v_\star \frac{\beta v_\star}{\beta v_\star \cosh \beta - \eta \sinh \beta},
\end{equation}
for
\begin{equation}
    \beta = \frac{1}{v_\star} \sqrt{\xi^2 + \eta^2}.
\end{equation}

This gives expressions for $u$ and $v$ in terms of the geodesic normal coordinates $\xi$ and $\eta$. One can see by Taylor expansion, that the coordinates $\xi=\eta=0$ indeed corresponds to the point $x^a = (u_\star, v_\star)$.
For small perturbations, we find that
\begin{equation}
    \label{eq:uv_approx}
    u - u_\star \approx \xi + \mathcal{O}({\rm quadratic}), \quad
    v - v_\star \approx \eta + \mathcal{O}({\rm quadratic}),
\end{equation}
which will be useful later.

\subsection{Inverting the coordinate transformation}

In the previous section we found $u$ and $v$ as a function of $\xi$ and $\eta$. In this section, we find the opposite: $\xi$ and $\eta$ as a function of $u$ and $v$.
We begin by taking $u - u_\star$ and dividing by our expression for $v$, one can obtain the relation
\begin{equation}
    \frac{\mu - \mu_\star}{\sigma} = \xi \frac{\sinh \beta}{\beta}.
\end{equation}
The left hand side of this expression is reminiscent of the $z$-score, or standardised Gaussian variable. For small $\beta$, $\sinh \beta / \beta \approx 1$, and thus $\xi$ is approximately equal to this standardised variable when one is close to the fiducial point.
The rest of the derivation requires us to consider separate cases, depending on the values of $\xi$ and $\beta$, which we outline below before summarising the result.

\paragraph{Case 1: $\xi \neq 0$, $\beta \neq 0$}

If we rearrange our expression for $v$ and assume $\xi \neq 0$, and substitute in this result to remove $\sinh \beta$, we obtain
\begin{equation}
    \beta v_\star \cosh \beta - \frac{\eta \beta}{\xi} \frac{\mu - \mu_\star}{\sigma} = \frac{\beta v_\star^2}{v} = \beta v_\star \frac{\sigma_\star}{\sigma}.
\end{equation}
Assuming $\beta \neq 0$, this can be rearranged to obtain
\begin{equation}
    \cosh \beta = \frac{\sigma_\star}{\sigma} + \frac{\mu - \mu_\star}{v_\star \sigma} \frac{\eta}{\xi}. 
\end{equation}
These can be combined using a hypergeometric identity to obtain
\begin{equation}
    \label{eq:xi_eta_from_mu_sigma}
    \frac{\eta}{\xi} = 
    \frac{\frac{1}{2} \left( \mu - \mu_\star \right)^2 + \sigma^2 - \sigma_\star^2}{\sigma_\star \left( \mu - \mu_\star \right) \sqrt{2}},
\end{equation}
where we have used that $v_\star = \sqrt{2}$.
We can now substitute this into our earlier expression for $\sinh \beta$
\begin{equation}
    \sinh \beta = 
    \frac{\beta}{\xi} \frac{\mu - \mu_\star}{\sigma}
    = \frac{\sgn(\xi)}{v_\star} \sqrt{1 + \frac{\eta^2}{\xi^2}} \frac{\mu - \mu_\star}{\sigma}.
\end{equation}
Since $\beta > 0$, we require the condition $\sgn (\xi) = \sgn (\mu - \mu_\star)$. But given that $\sgn(\sinh \beta / \beta) = 1$, from the earlier formula, we know that this must be true. We therefore find that we can obtain $\beta$ from
\begin{equation}
    \label{eq:beta_from_mu_sigma}
    \sinh \beta =
    \frac{\sqrt{
    \left( \frac{1}{2} \left( \mu - \mu_\star \right)^2 + (\sigma - \sigma_\star)^2 \right)
    \left( \frac{1}{2}  \left( \mu - \mu_\star \right)^2 + (\sigma + \sigma_\star)^2 \right)
    }}{2 \sigma_\star \sigma},
\end{equation}
and thus we now have a way of computing $(|\xi|, |\eta|)$ from $(\mu, \sigma)$.
We note that the signs of the coordinates can be determined using
\begin{equation}
    \label{eq:signs_from_mu_sigma}
    \sgn(\xi) = \sgn(\mu - \mu_\star), \quad
    \sgn(\eta) = \sgn \left( \frac{1}{2} \left( \mu - \mu_\star \right)^2 + \sigma^2 - \sigma_\star^2  \right).
\end{equation}

\paragraph{Case 2: $\xi = 0$, $\beta \neq 0$}

If $\xi=0$, then automatically we have that $\mu = \mu_\star$ and that $\beta v_\star = |\eta|$. We therefore obtain
\begin{equation}
    v = v_\star \frac{1}{\cosh \beta - \sgn(\eta) \sinh \beta} 
    = v_\star \exp \left( \sgn(\eta) \beta\right),
\end{equation}
and thus (given that $\beta \geq 0$)
\begin{equation}
    \beta = \left| \log \left( \frac{\sigma}{\sigma_\star} \right) \right|,
\end{equation}
where $\eta > 0$ for $\sigma > \sigma_\star$ and  $\eta < 0$ for $\sigma < \sigma_\star$.
Note that this is identical to \cref{eq:beta_from_mu_sigma} for $\mu = \mu_\star$.

\paragraph{Case 3: $\beta = 0$}

It looks like there are two remaining cases to consider: $\xi \neq 0$ with $\beta = 0$, and $\xi = 0$, $\beta = 0$, but from the definition of $\beta$, the first of these is impossible. We also note that the second of these means that $\xi=\eta=0$. We already noted that this corresponds to the point $\mu = \mu_\star$ and $\sigma = \sigma_\star$, so there is nothing else to add here.


\paragraph{Summary}

In summary, to obtain $(\mu, \sigma)$ from $(\xi, \eta)$:
\begin{enumerate}
    \item If $\mu = \mu_\star$ and $\sigma = \sigma_\star$, then $\xi = \eta = 0$.
    \item Otherwise, $\beta$ can be obtained from evaluating \cref{eq:beta_from_mu_sigma}. The signs of $\xi$ and $\eta$ are given by \cref{eq:signs_from_mu_sigma}. To obtain the magnitudes of $\xi$ and $\eta$, one can use the results:
        \begin{enumerate}
            \item If $\mu = \mu_\star$, then $|\eta| = v_\star \beta$ and $\xi = 0$.
            \item Otherwise, use \cref{eq:xi_eta_from_mu_sigma}
        \end{enumerate}
\end{enumerate}

\subsection{Metric and Jacobian for Geodesic Normal Coordinates}

Now we have functions for $u$ and $v$ as a function of $\xi$ and $\eta$, let us find the metric in these new coordinates.
We first find the Jacobian matrix
although we do not write the explicit components here, as they are not particularly illuminating.
One important property, however, is its determinant, which is
\begin{equation}
    \det J 
    = \frac{v_\star^2 \beta \sinh \beta}{\left( \beta v_\star \cosh \beta - \eta \sinh \beta\right)^2}
    = \frac{\mu - \mu_\star}{\sigma_\star} \frac{\sigma}{\sigma_\star} \frac{1}{\xi}.
\end{equation}

We see here that, unless $\mu = \mu_\star$, $\xi = 0$ or $\xi \to \infty$, we will definitely have a finite and non-zero determinant for finite values of $\mu$ and $\sigma$ (provided $\sigma \neq 0$ and $\sigma_\star \neq 0$, but we always have this case). Let us investigate the case $\mu = \mu_\star$ to see if this is problematic. This case corresponds to $\xi = 0$ (so our first two problematic cases are identical), and thus we must consider the limit of $(u - u_\star) / \xi$ as $\xi \to 0$.
In this limit $\beta = \eta / v_\star$, and hence
\begin{equation}
    \lim_{\xi \to 0} \frac{u - u_\star}{\xi} 
    = 
    \frac{v_\star \sinh \left( \eta / v_\star \right)}{\eta \cosh  \left( \eta / v_\star \right) - \eta  \sinh \left( \eta / v_\star \right)}
    = \left[ \frac{\eta}{v_\star} \left( \coth \left( \frac{\eta}{v_\star} \right) - 1 \right) \right]^{-1}.
\end{equation}
The function $y = (x (\coth x - 1))^{-1}$ has no roots and is always positive, hence we find that the Jacobian is always positive at the point $u = u_\star$.
Therefore, unless $\xi$, $\mu - \mu_\star$ or $\sigma$ is infinite, then we always have a finite and non-zero Jacobian, and thus we can locally invert the coordinate transformation.

\subsubsection{Metric}

We can use the Jacobian to obtain the metric (i.e. the Fisher) in our new coordinates, which is
\begin{equation}
    \tilde{F}_{ab} = 
    \tilde{g}_{ab} =
    \frac{1}{2 \beta^4}
    \begin{pmatrix}
        \eta^2 \sinh^2 \beta + \xi^2 \beta^2 & \eta \xi \left( \beta^2 - \sinh^2 \beta \right) \\
        \eta \xi \left( \beta^2 - \sinh^2 \beta \right) & \eta^2 \beta^2 + \xi^2 \sinh^2 \beta
    \end{pmatrix},
\end{equation}
where we have used that $v_\star = \sqrt{2}$. To verify that this has the desired behaviour, let us expand the metric around the point $\xi=\eta=0$. In this case, we obtain
\begin{equation}
    \tilde{g}_{ab} \approx \mat{1} 
    + \frac{1}{6} 
    \begin{pmatrix}
        \eta^2 & -\eta \xi \\
        - \eta \xi & \xi^2
    \end{pmatrix}
    + \mathcal{O} \left( |x|^3 \right),
\end{equation}
so indeed, in these coordinates, the metric is flat at zeroth and linear order, with the correction occurring at quadratic order.
Using the series expansion of \cref{eq:uv_approx}, this can be written as
\begin{equation}
    \tilde{g}_{ab} \approx \mat{1} 
    + \frac{1}{6} 
    \begin{pmatrix}
        \left( v - v_\star \right)^2 & -\left( u - u_\star \right) \left( v - v_\star \right) \\
        -\left( u - u_\star \right) \left( v - v_\star \right) & \left( u - u_\star \right)^2
    \end{pmatrix}
    + \ldots
\end{equation}
which is identical to \cref{eq:gaussian_metric_expansion} (since $v_\star = \sqrt{2}$), as required.

\subsection{Loss function for neural network}


In the neural network search for our coordinates, we compute the Frobenius norm (\cref{eq:flat_loss}) between the metric and the identity.
For the geodesic normal coordinates we find that this takes a particularly nice form
\begin{equation}
    \big\|\tilde{g}_{ab} - \delta_{ab}\big\|_{\mathcal{F}}^{2}
    = \left( \frac{\sinh^2 \beta}{\beta^2} - 1 \right)^2
    \approx \frac{\beta^4}{9} + \mathcal{O}(\beta^6),
\end{equation}
where the first equality is exact, and the approximation is valid for $\beta \ll 1$.
The quartic behaviour in the square of the Frobenius norm is consistent with geodesic normal coordinates having quadratic deviations from flatness near the fiducial point.

\section{Experiment Details}
\label{app:experiments}
\paragraph{Default Distillery setup.} The distillery comes with controls that users can tune for specific problems. We calibrated our setup on the Rosenbrock problem, and made slight modifications (such as changing network size or depth) for each application. A fixed simulation budget is chosen in advance. The simulations are $50$--$50$ split into Fishnets training and validation datasets. Once the Fisher matrices are learned, the flattening network is fit on the \textit{validation} set of true parameters and ensemble of learned Fishers $(\theta, \{\mat{F}\}_j)$ (the train split is discarded). {\textbf{SR data augmentation.} Once the flattening ensemble is trained, the map $\gen{\bm{\eta}}(\negv{\bm{\theta}})$ no longer depends on the underlying simulator; we therefore draw a fresh, dense, uniform sample from the prior support (typically ${\sim}5$--$10\times$ the Fishnet validation count, e.g.\ ${\sim}2{,}000$ Rosenbrock draws against ${\sim}250$ validation simulations) and pass it through the ensemble to obtain the augmented triple $(\negv{\bm{\theta}}, \bar{\gen{\bm{\eta}}}, \delta\gen{\bm{\eta}})$ used as SR inputs. Empirically this dense, well-conditioned coverage of the prior is the single most reliable lever for clean Pareto fronts and recovery of the analytic coordinates---using only the ${\sim}250$ original validation points routinely yields noisier expressions and slower convergence.} Each ensemble member is aligned to a reference member using the Procrustes alignment scheme. Symbolic regression is carried out for 120 seconds per component, and expressions are pruned with an importance-weighted Frobenius loss threshold of $0.1$.

\paragraph{Downstream neural posterior estimation.}
For each of the four scientific applications and the Rosenbrock
validation, we train a Masked Autoregressive Flow (MAF) as a Neural
Posterior Estimator (NPE) using \texttt{ltu-ili} with the \texttt{lampe}
backend \citep{papamakarios2021normalizing}. We
use $5$ MAF transforms, $50$ hidden features per transform, two MAF
repeats, learning rate $10^{-4}$, and up to $1\,000$
epochs, with early stopping by smoothed validation log-probability over a
$10$-epoch moving window. The simulation-budget sweeps span
$\{100, 500, 1\,000, 5\,000, 10\,000\}$ training simulations with corresponding batch sizes $[8, 22, 32, 32, 32]$, and average
over three random seeds. Crucially, the simulations are generated uniformly in $\theta$-space and, for the $\eta$-coordinate inference, converted via the discovered coordinate transformation. This ensures that prior volumes do not influence the inference results.  Log-losses in $\eta$-space are converted to $\theta$-space using the mean
$\log\!\big|\!\det\!\big(\partial\theta/\partial\eta\big)\big|$ over the
prior, estimated by central finite differences on $512$ samples.

\subsection{Rosenbrock}
\label{app:rosenbrock}

Consider a Gaussian in two dimensions controlled by means ${\bm{\eta}} = (\nu_0, \nu_1)$, which we would like to infer from data, and fixed diagonal covariance 
$\gmat{\Sigma} = \text{diag}(\sigma_0^2, \sigma_1^2)$. 
Assuming a uniform prior on $\theta$, the posterior distribution can be expressed as
\begin{equation}\label{eq:post_flat}
    \log p({\bm{\eta}} | \{\textbf{x}^i \}) = -\frac{1}{2} \frac{(\nu_0 - \langle x_0 \rangle)^2}{(\sigma_0 / \sqrt{n_d})^2} - \frac{1}{2} \frac{(\nu_1 - \langle x_1 \rangle)^2}{(\sigma_1 / \sqrt{n_d})^2},
\end{equation}
where $\langle a \rangle = \frac{1}{n_d}\sum_{i=1}^{n_d} a^i$, which is also Gaussian, and $n_{\rm d}$ is the number of data points.
This likelihood has a known Fisher information matrix: $\mat{F}^\prime = n_d \diag(1/\sigma_0^2, 1/\sigma_1^2)$.
Hence, in this coordinate system, by a simple rescaling we can make the Fisher information the identity.

Let us now consider a new coordinate system for which the Fisher information is not so clearly related to the identity. Let us define $\bm{\theta} = (\mu_0,\mu_1)$ such that
\begin{equation}
    \label{eq:nu_from_mu}
    \mu_0 = \nu_0, \quad \mu_1 = \nu_1 + \nu_0^2,
    \quad \Leftrightarrow \quad
    \nu_0 = \mu_0, \quad \nu_1 = \mu_1 - \mu_0^2.
\end{equation}
Applying the Jacobian of this transformation (which has determinant $1$), the
posterior in the new coordinates is given as
\begin{equation}\label{eq:post_rosen}
    \log p(\bm{\theta} | \{\textbf{x}^i \}) = -\frac{1}{2}\frac{(\mu_0 - \langle x_0 \rangle)^2}{(\sigma_0 / \sqrt{n_d})^2} - \frac{1}{2} \frac{(\mu_1 - \mu_0^2 - \langle x_1 \rangle)^2}{(\sigma_1 / \sqrt{n_d})^2},
\end{equation}
which is a non-convex Rosenbrock-like function \citep[e.g.][]{rosenbrock_automatic_1960}. In these coordinates, the Fisher matrix is given by
\begin{equation}
    \mat{F} = n_d \begin{pmatrix}
        \frac{1}{\sigma^2_0} + \frac{4 \mu_0^2}{\sigma_1^2} & \frac{-2\mu_0}{\sigma_1^2} \\
        \frac{-2\mu_0}{\sigma_1^2} & \frac{1}{\sigma_1^2}
    \end{pmatrix}.
\end{equation}

We have therefore shown that a Rosenbrock-like posterior (\cref{eq:post_rosen}), which has a complicated Fisher matrix, can be transformed into a Gaussian posterior (\cref{eq:post_flat}) with a simple Fisher matrix. Since this result is exact, we use it as a benchmark to determine whether our methodology can find the coordinates ${\bm{\eta}}$ given simulations performed in the coordinates $\bm{\theta}$.

To test the algorithm, we generate $500$ simulations over a fixed uniform prior 
$(\mu_0, \mu_1) \sim \mathcal{U}[-3,3]\times[-3, 3]$, splitting 50-50 into train and test sets.
These coordinates are then transformed to $(\nu_0, \nu_1)$ via \cref{eq:nu_from_mu}. 
We then generate $n_{\rm d}=50$ two-dimensional data points from a multivariate normal 
$\{ \textbf{x} \}_{i=1}^{n_d} \sim \mathcal{N}([\nu_0, \nu_1], \gmat{\Sigma} )$ with fixed, diagonal covariance $\gmat{\Sigma} = \mathrm{diag}({1.0, 2.0})$. For inference scaling we investigate $\gmat{\Sigma} = \mathrm{diag}({4.0, 1.0})$, which exacerbates the ``banana'' shape of the posterior, deviating from a Gaussian approximation.

\begin{figure}
    \centering
    \includegraphics[width=0.5\linewidth]{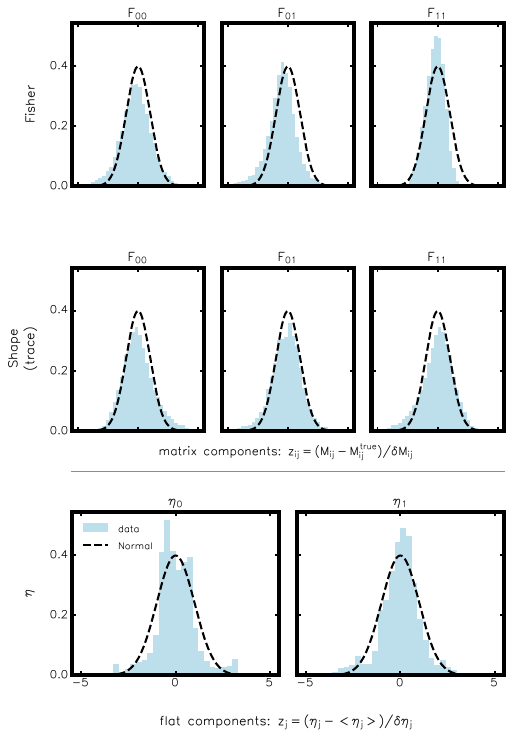}
    \caption{\textit{Top panels:} Fishnets provides an unbiased estimate of a) unique Cholesky factors and b) normalised shape of the true Rosenbrock Fisher Matrix. \textit{Bottom panel:} Learned flat coordinates $\eta$ and their ensembled-estimated uncertainties $\delta \eta$  approximately follow a Gaussian sampling distribution.}
    \label{fig:rosen-sampling-dist}
\end{figure}

\subsection{Gaussian with unknown variance}
We generate $10\,000$ simulations via $\mu \sim \mathcal{U}[-5,5]$ and
$\sigma \sim \mathcal{U}[0.2, 7.0]$, and generate
$x \sim \mathcal{N}(\mu, \sigma)$ with $\dim(x) = 50$. We adopt the same
network setup calibrated on the Rosenbrock example. 
\textbf{Results.} Using the same distillery settings as the Rosenbrock example, we recover the MDL expressions
\begin{align*}
    \eta_1 &= \mu\,(0.4\,\sigma - \exp(0.14\,\sigma))\,(0.06\,\mu^2 - 0.02\,\exp(0.95\,\sigma) - 3.0) + 7.0, \\
    \eta_2 &= 0.8\,\sigma + (0.4\,\mu^2 + 1.0\,\sigma)^{\exp(-0.13\,\sigma)} - 0.9
\end{align*}
and Frobenius expressions
\begin{align*}
  \eta_1 &= \bigl(\mu\,(0.06\,\mu^2 - 3.0)\,(0.02\,\sigma\,\exp(0.28\,\sigma) - 1.0) + (0.3\,\mu + 7.0)\,\exp(0.28\,\sigma)\bigr)\,\exp(-0.28\,\sigma), \\
  \eta_2 &= 0.8\,\sigma + (0.4\,\mu^2 - 0.04\,\mu + \sigma)^{\exp(-0.12\,\sigma)} - 0.9.
\end{align*}

\subsection{SIR Simulation Details}
\label{app:sir-simulation-details}
\begin{figure}[h!]
    \centering
    \includegraphics[width=\linewidth]{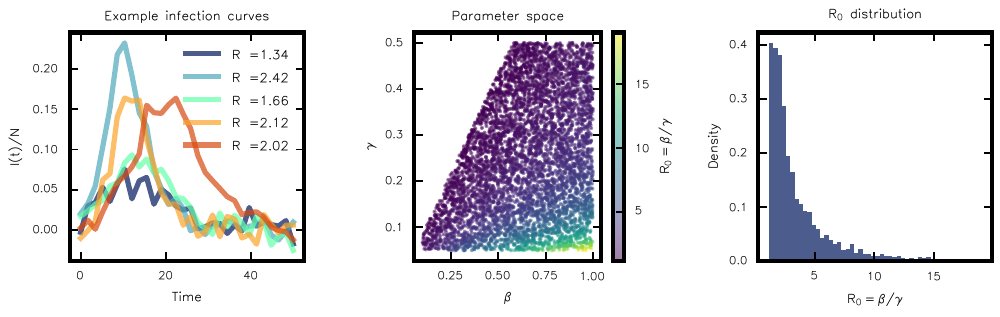}
    \caption{Example SIR epidemic ($R_0>1.0$) time-series data and parameter values, coloured by $R_0$.}
    \label{fig:sir-data-illustration}
\end{figure}
We use the classical susceptible--infectious--removed (SIR) epidemic model
\citep{kermack1927contribution,hethcote2000mathematics}, in which a closed
population of size $N$ is partitioned into susceptible $S(t)$, infectious
$I(t)$, and removed/recovered $R(t)$ compartments:
\begin{align}
\frac{{\rm d}S}{{\rm d}t} &= -\beta \frac{SI}{N}, &
\frac{{\rm d}I}{{\rm d}t} &= \beta \frac{SI}{N} - \gamma I, &
\frac{{\rm d}R}{{\rm d}t} &= \gamma I,
\end{align}
where $\beta$ is the transmission rate, $\gamma$ the recovery rate, and the basic reproduction number is $R_0=\beta/\gamma$.

\begin{figure}
    \centering
    \includegraphics[width=\linewidth]{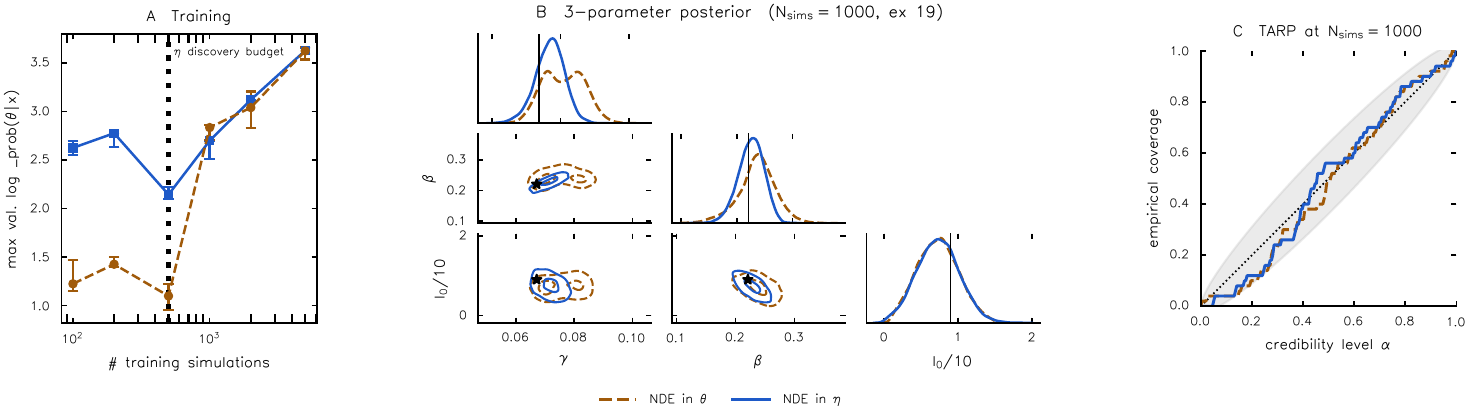}
    \caption{Supplementary details for the SIR experiment. \textit{Left:} Discovered coordinates require fewer simulations in noisier settings to achieve the same log-probability scores in downstream NDE posterior estimation in $\theta$ space. \textit{Middle:} 3D posteriors (here for $N_{\rm sim}=1000$) show good agreement with simulation truths and tighter $\eta$ constraints. \textit{Right:} both $\eta$ and $\theta$ posteriors indicate underconfident (conservative) coverage across test simulations.}
    \label{fig:sir_appendix}
\end{figure}

{Simulations use $N=5000$ with initial conditions $I(0)\sim\mathrm{Pois}(\lambda=8)$, $S(0)=N-I(0)$, $R(0)=0$. We observe the normalised infectious trajectory $I(t)/N$ at $30$ evenly spaced times over $t\in[0,50]$ with additive Gaussian noise of standard deviation $\sigma_{\rm obs}=0.05$. Parameters are sampled independently as $\beta\sim\mathcal{U}(0.1,1.0)$, $\gamma\sim\mathcal{U}(0.05,0.5)$, with the rescaled initial condition $I_0/10$ entering as a third inference coordinate; we generate $2\,500$ training and $2\,500$ test simulations.  For the epidemic-regime analysis we retain samples with $R_0 \ge 1+\delta$ with $\delta=0.15$, and discard the band $0.85 \le R_0 < 1.15$; the corresponding non-epidemic subset is $R_0 < 0.85$.} We display example simulations and parameters in Fig. \ref{fig:sir-data-illustration}. 

To demonstrate information capture in the identified $R_0$ direction, we grade posterior values of $R_0$, computed from posterior chains in $(\beta, \gamma)$. For each test observation $x^{(i)}$ with true reproduction number
$R_0^{(i)}=\beta^{(i)}/\gamma^{(i)}$, we score the posterior $q(R_0\mid x^{(i)})$
via the continuous ranked probability score (CRPS)
\begin{equation}
\mathrm{CRPS}\!\left(q,\,R_0^{(i)}\right)
\;=\;\mathbb{E}_{R\sim q}\,\bigl|R-R_0^{(i)}\bigr|
\;-\;\tfrac{1}{2}\,\mathbb{E}_{R,R'\sim q}\,\bigl|R-R'\bigr|,
\label{eq:crps-R0}
\end{equation}
and report the mean over the test set, $\overline{\mathrm{CRPS}}(R_0)=N^{-1}\sum_i \mathrm{CRPS}(q,R_0^{(i)})$.
CRPS is a strictly proper scoring rule with units of $R_0$: it
reduces to the absolute error when $q$ is a point mass, and it is uniquely
minimized in expectation when $q$ matches the true posterior over $R_0$
\citep{gneiting2007strictly}. We estimate~\eqref{eq:crps-R0} from $M$ posterior
samples $\{R_j\}_{j=1}^{M}$ using the unbiased $U$-statistic
\begin{equation}
\widehat{\mathrm{CRPS}}
\;=\;\frac{1}{M}\sum_{j}\bigl|R_j-R_0^{(i)}\bigr|
\;-\;\frac{1}{2M(M{-}1)}\sum_{j\neq k}\bigl|R_j-R_k\bigr|,
\end{equation}
with $R_j=\beta_j/\gamma_j$ obtained from posterior draws $(\beta_j,\gamma_j)\sim q(\,\cdot\mid x^{(i)})$.

\subsection{Weak Gravitational Lensing}
\label{app:WL}

For our weak lensing example we adopt the $5\,000$ noisy tomographic WL simulations from \citet{Lucas_Makinen_hybrid_2025}, generated with the \texttt{pmwd} particle-mesh code \citep{li_pmwd_2022} and collected into four redshift bins to form convergence-image data of shape $(128,128,4)$. {Simulations span a wide uniform prior $p(\Omega_m, S_8) = \mathcal{U}([0.15, 0.7] \times [0.35, 1.52])$ and additionally vary the initial-condition seed; following \citet{Lucas_Makinen_hybrid_2025} we histogram all auto- and cross-power spectra per redshift bin into a 60-component summary per simulation. Additive shape noise with variance $[0.00045021, 0.00087473, 0.00134725, 0.00183411]$ per tomographic bin is injected for downstream neural posterior estimation (Fig. \ref{fig:wl-ex-post})} The \textit{same} underlying dark matter simulations are used in the distillery and inference steps.

\begin{figure}[t]
    \centering

    \begin{subfigure}[t]{0.48\textwidth}
        \centering
        \includegraphics[width=\linewidth]{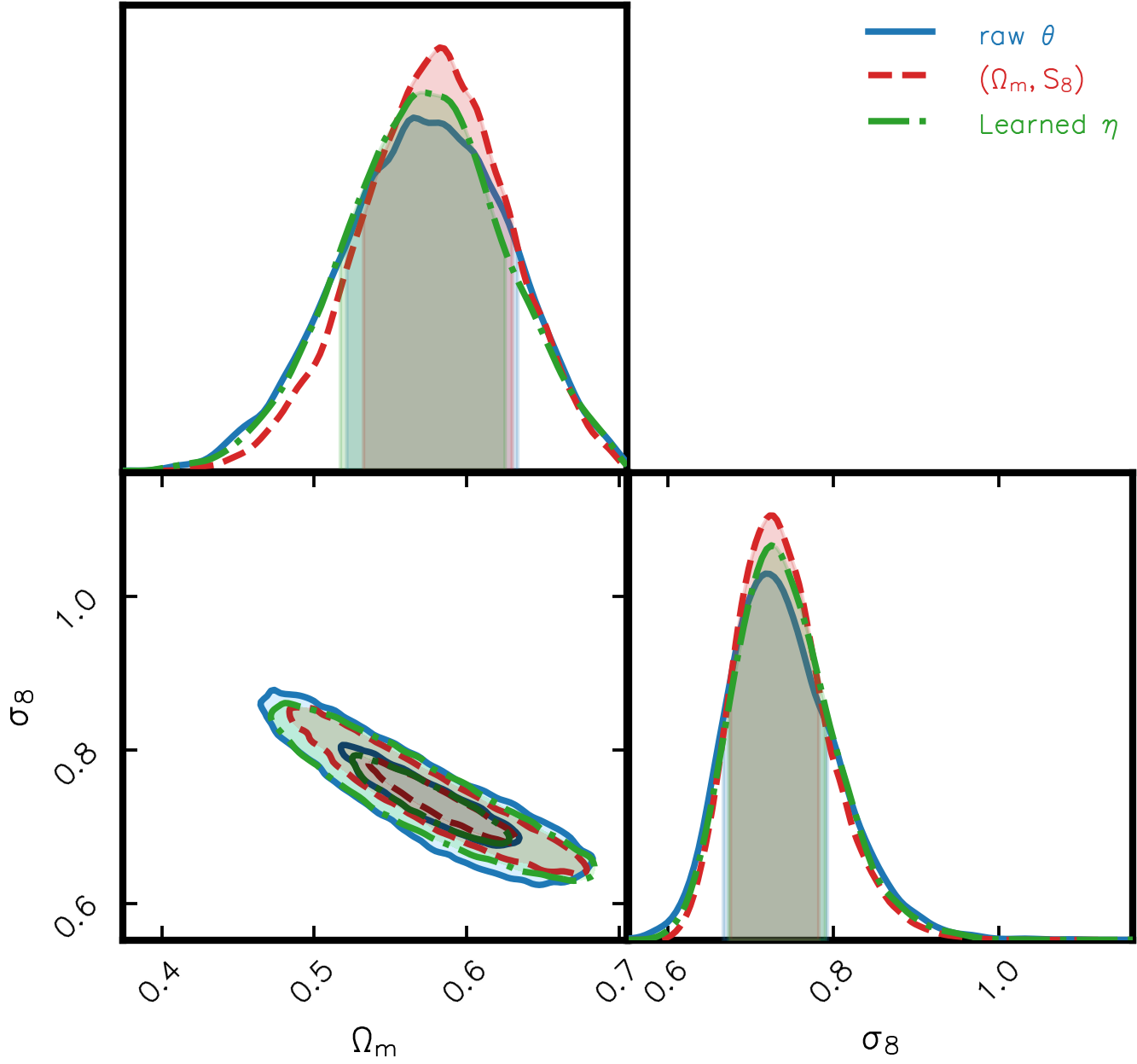}
        \label{fig:panel-a}
    \end{subfigure}
    \hfill
    \begin{subfigure}[t]{0.48\textwidth}
        \centering
        \includegraphics[width=\linewidth]{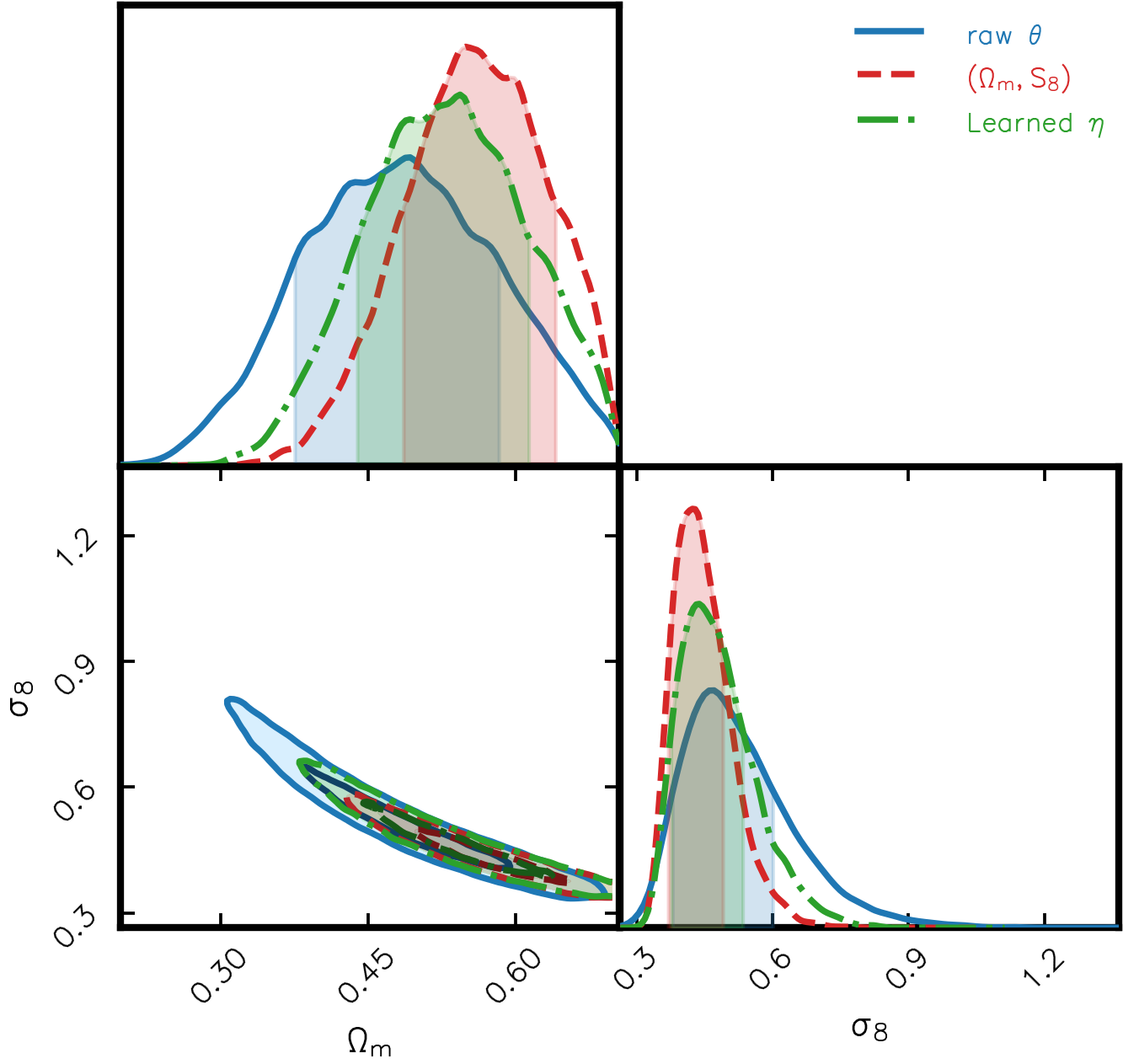}
        \label{fig:panel-b}
    \end{subfigure}

    \caption{Example NPE posteriors using learned versus conventional weak lensing coordinates. Learned coordinates trace the conventional coordinate constraints, and outperform raw $\theta$ sampling.}
    \label{fig:wl-ex-post}
\end{figure}



\subsection{Heater experiment}
\label{app:heater}
{We expand on the rank-1 industrial-heater experiment of \cref{sec:heater}, repeating the simulator and prior for self-containment and providing the hyperparameter and ablation details deferred from the main text.}

We construct a minimal synthetic experiment to demonstrate that the
Distillery can perform genuine \emph{dimensionality reduction}, leading to immense performance gains in high dimensions. A first-order resistive-heater
plant is driven by an applied voltage and current $\theta = (V, I)$, with observable 20-point time series of the temperature step response (see e.g. \cite{Incropera2011}),
\begin{equation}
  y(t) \;=\; (V \cdot I)\,(1 - e^{-t/\tau}) \;+\; \varepsilon(t),
  \qquad \varepsilon(t) \overset{\text{iid}}{\sim}
         \mathcal{N}(0, \sigma^{2}).
  \label{eq:heater_simulator2}
\end{equation}
The plant dynamics depend on $(V, I)$ \textit{only} through the dissipated power
$P = V \cdot I$, so the Fisher information matrix
$F_\theta(\theta) = (\|k\|^{2}/\sigma^{2})\,(\theta_2, \theta_1)^{\top}
(\theta_2, \theta_1)$ is structurally rank-1 \emph{everywhere}: $\theta$
parameterises a 2-D box but the data manifold is exactly 1-D. This is a common occurrence in industrial control problems where two physical dials multiply together to produce a single output. 

\paragraph{Scaling Impact.} Density estimation becomes increasingly difficult in high dimensions, requiring $N(d, \varepsilon) \;\sim\; \varepsilon^{-(2\beta + d)/(2\beta)}$ simulations to reach an error tolerance $\varepsilon$ (Stone's relation; \cite{stone1980}). If the data manifold
has intrinsic dimension $d^{*} < d$, training on the distilled
coordinates costs only $N(d^{*}, \varepsilon)$, giving an
exponential-in-$k$ saving
\begin{equation}
  \frac{N(d, \varepsilon)}{N(d^{*}, \varepsilon)}
  \;=\; \varepsilon^{-k/(2\beta)},
  \qquad k \equiv d - d^{*}.
  \label{eq:distillation_gain}
\end{equation}
For $\beta = 1$ (Lipschitz) and $\varepsilon = 10^{-2}$ this varies as
$10^{k/2}$. Explicit symbolic (or neural) distillation of the \textit{intrinsic} parameter combinations delivers this exponential gain in density estimation by effectively shrinking the parameter space. \textbf{Empirical scaling sweep.}
To make this explicit in the context of neural posterior estimation (NPE; \cite{papamakarios2021normalizing,pydelfiAlsing_2019}), we extend the heater problem to a chain-product simulator
\begin{equation}
  y(t) \;=\; \Big(\prod_{i=1}^{d} \theta_i\Big)\,
             (1 - e^{-t/\tau}) \;+\; \varepsilon(t),
  \qquad \theta_i \sim \mathcal{U}[1, 2],
\end{equation}
in which the intrinsic dimension is fixed at one for every $d$.
Figure~\ref{fig:mu-sigma-heater-scaling} shows the best validation
log-probability of two neural density estimators
trained on a fixed budget of $N_{\text{sim}} = 10^{3}$ simulations: the
raw flow with target $\theta \in \mathbb{R}^{d}$ and the distilled flow
with target $\eta = \prod_i \theta_i \in \mathbb{R}$. The raw flow's
inference quality degrades monotonically with $d$, tracing Stone's relation, while the distilled
flow is, by construction, insensitive to $d$. The gap widens at the rate
predicted by Eq.~\eqref{eq:distillation_gain}, despite the underlying
observed information being identical at every $d$.

\paragraph{{Simulator and prior.}}
{The simulator is the resistive-heater step response \cref{eq:heater_simulator} with thermal time constant $\tau = 1$, observation horizon $t \in [0, 4]$ sampled at $n_t = 20$ uniformly spaced points, and per-sample noise standard deviation $\sigma = 0.2$ ($\sigma^{2} = 0.04$ in temperature units). The prior is $\theta = (V, I) \sim \mathcal{U}[1, 2]^{2}$---a tight factor-of-two box around the operating point---with $N_{\text{sim}} = 10^{3}$ i.i.d.\ training draws and an independent $10^{3}$ for validation. The Fisher information matrix
$F_\theta(\theta) = (\|k\|^{2}/\sigma^{2})\,(\theta_2, \theta_1)^{\top}(\theta_2, \theta_1)$
is structurally rank-1 everywhere because the dynamics depend on $(V, I)$ only through the dissipated power $P = V \cdot I$.}

\paragraph{Fishnet ensemble.}
We use an ensemble of $10$ Fishnets, each a
$3$-layer MLP with hidden widths drawn uniformly in $[64, 128]$, trained
for $200$ epochs at learning rate $10^{-3}$ and batch size $64$ with
$\mathrm{patience} = 30$ on the validation MLE+Fisher objective. The
ensemble-mean Fisher at the prior centre is
\begin{equation*}
  \bar{F}_\theta \;\approx\;
  \begin{pmatrix} 0.62 & 0.58 \\ 0.58 & 0.62 \end{pmatrix},
  \qquad
  \mathrm{eigvals}(\bar{F}_\theta) \;\approx\; (0.033,\; 1.20),
\end{equation*}
giving an empirical condition number of $\sim 36$ -- consistent with a
structurally rank-1 Fisher whose minor eigenvalue is set by the
finite-data noise floor at this $\sigma$.

\paragraph{Flattening flow.}
We fit a single flattening network $\eta = f_{\phi}(\theta)$ with the log-Frobenius loss, {trained for $200$ epochs at learning rate $10^{-3}$ on the full $10^{3}$-sample training set. The asymptotic loss $\mathcal{L} \approx 0.79$ is the structural rank-1 floor of $\log\|Q^{-1} - I\|_F$: the unidentifiable direction has eigenvalue zero exactly, so a higher-rank target would drive the loss to zero. The Jacobian determinant $|\det J| \approx 0.04$ indicates that the 2D problem contains only one meaningful direction.}

\paragraph{Symbolic regression.}
{Discovered coordinates are post-processed with the standard Distillery pipeline. The two SR-recovered expressions are
\begin{align}
  \eta_{\text{nuis}}(V, I) &\;\approx\;
    +0.10\,V \;-\; 0.09\,I \;+\; 0.4
    \;\approx\; 0.10\,(V - I) + 0.4, \\
  \eta_{\text{iden}}(V, I) &\;\approx\;
    -0.60\,V \;-\; 0.60\,I \;+\; 6.0
    \;=\; -0.60\,(V + I) + 6.0,
\end{align}
with $|r|(\eta_{\text{iden}}, V{+}I) = 1.00$ and $|r|(\eta_{\text{nuis}}, V{-}I) = 1.00$ (Table~\ref{tab:heater_correlations}); cross-talk into the identifiable subspace is below $0.10$.}

\begin{table}[h]
  \centering
  \begin{tabular}{lcc}
    \toprule
    candidate axis & $|r(\eta_{\text{nuis}}, \cdot)|$
                   & $|r(\eta_{\text{iden}}, \cdot)|$ \\
    \midrule
    $V \cdot I$ \quad (true identifiable axis)
                                & $\sim 0.10$ & $\mathbf{0.953}$ \\
    $\sqrt{V \cdot I}$          & $\sim 0.10$ & $\mathbf{0.977}$ \\
    $V + I$ \;(local linearisation of $\sqrt{VI}$)
                                & $\sim 0.08$ & $\mathbf{1.000}$ \\
    $\log(V / I)$ \quad (true nuisance axis)
                                & $\mathbf{0.968}$ & $< 0.025$ \\
    $V - I$                     & $\mathbf{0.996}$ & $< 0.025$ \\
    \bottomrule
  \end{tabular}
  \caption{Pearson $|r|$ of the discovered Distillery coordinates
  $\eta_{\text{nuis}}, \eta_{\text{iden}}$ against ground-truth physics
  axes for the rank-1 heater experiment, evaluated over the validation
  set. Boldface marks the highest correlation in each row. The two
  Distillery coordinates partition the rank-1 manifold into a clean
  identifiable-and-nuisance pair with cross-talk $|r| < 0.10$.}
  \label{tab:heater_correlations}
\end{table}

The recovered $V+I$ coordinate is a local approximation to $V\cdot I$ within the prior box $[1,2]^2$.

\paragraph{NPE-MAF training (sweep).}
For the chain-product sweep
(Fig.~\ref{fig:mu-sigma-heater-scaling}), we train a single
MAF~\citep{papamakarios2017} with $5$ transforms and hidden width $50$ for raw coordinates, and a 4-component MDN for 1D distilled coordinates using the
\texttt{ltu-ili}~\citep{lemos2024ltu} \texttt{lampe} backend with learning rate $10^{-4}$, batch size
$64$, patience $30$, max epochs $1000$, and $3$ independent random
seeds per ambient dimension $d \in \{2, 3, 4, 6, 8, 10, 12\}$. Both raw
and distilled flows see identical $(\theta, y)$ pairs and identical
training budgets. The only difference is the regression target
($\theta \in \mathbb{R}^{d}$ vs $\eta = \prod_i \theta_i \in \mathbb{R}$) and NPE head parametrisation.

\subsection{Gravitational-wave inspiral waveforms (TaylorF2 warm-up)}
\label{app:gw-details}
We use TaylorF2 inspiral waveforms with non-spinning 2PN phase. The
frequency grid runs from $f_{\rm low} = 20$\,Hz to the lightest-system
ISCO frequency $f_{\rm ISCO} = (6^{3/2}\pi (m_1+m_2) M_\odot)^{-1}$ in
steps of $\Delta f = 0.5$\,Hz, capped at $1\,024$\,Hz. Waveforms are
whitened with the analytic Advanced-LIGO design PSD
$S_n(f) = 10^{-49}(x^{-4.14} + 2 + 2x^2)$, $x = f / 215\,{\rm Hz}$, for
$f \geq 10$\,Hz, with whitening factor $\sqrt{4\,\Delta f / S_n(f)}$. Real
and imaginary parts of the whitened waveform are concatenated, and PCA is
fit on a bank of $1\,000$ noiseless whitened waveforms with $40$ retained
components. Unit Gaussian noise is added in the PCA-compressed space. Mass
priors are $m_1, m_2 \sim \mathcal{U}(5, 50)\,M_\odot$ subject to
$m_1 \geq m_2$.

\paragraph{Conventional baseline.}
For comparison we also train a neural posterior estimator in the
conventional $(\mathcal{M}_c, q)$ coordinates,
\[
\mathcal{M}_c = \frac{(m_1 m_2)^{3/5}}{(m_1+m_2)^{1/5}}, \qquad
q = m_2/m_1,
\]
with their analytic inverse used to convert posterior samples back to
$(m_1, m_2)$ for evaluation.

\begin{figure}
    \centering
    \includegraphics[width=\linewidth]{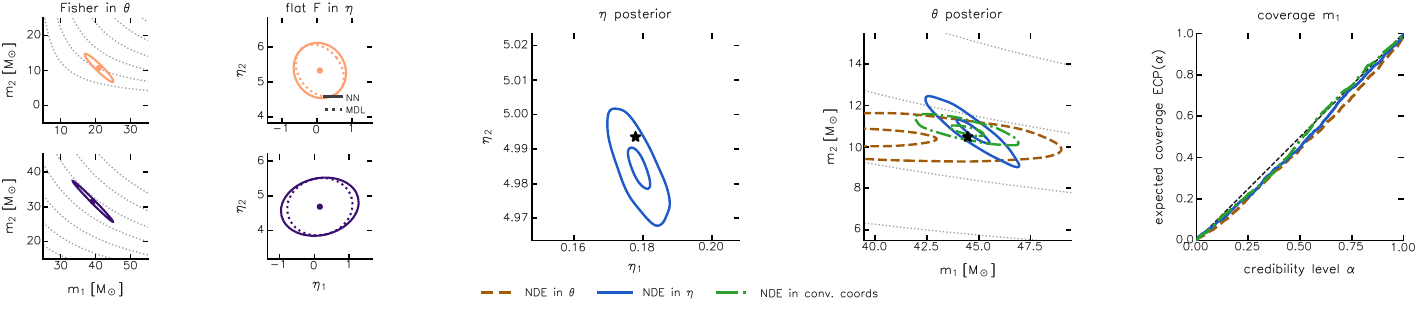}
    \caption{Benchmark GW result, \textsc{TaylorF2} (inspiral-only, 2PN, non-spinning). Same panel layout as \cref{fig:gw_imr}. The learned coordinates collapse onto a smooth bijection of the conventional chirp-mass / mass-ratio combination $(\mathcal{M}_c, q)$; posterior contours and PP-plots in $\eta$ and $(\mathcal{M}_c, q)$ are consistent with the analytic argument that the inspiral phase is dominated by $\mathcal{M}_c$ in this mass range.}
    \label{fig:gw_taylorf2}
\end{figure}

\paragraph{{Benchmarking results.}}
{For masses rescaled to $[0.1, 1.0]$, the distillery returns the symbolic coordinates of \cref{sec:applications},
$\eta_1 = 0.2\,m_1$, $\eta_2 = 5.5\,\exp(-0.08\,(m_1+m_2)^2)$, with analytic inverse
$m_1 = 5\eta_1$, $m_2 = -5\eta_1 + 2.5\sqrt{3.41 - 2\log|\eta_2|}$ used for downstream sampling. Posterior estimation in $\eta$-space matches or exceeds the chirp-mass coordinates at all simulation budgets tested, with the largest gains at low budgets.}

\subsection{Gravitational-wave inspiral--merger--ringdown
(\textsc{IMRPhenomD})}
\label{app:gw-imr-details}
{For the main GW result we replace the inspiral-only TaylorF2 model with the phenomenological inspiral--merger--ringdown waveform \textsc{IMRPhenomD} \citep{Khan2016IMRPhenomD,Husa2016IMRPhenomD}, generated through \texttt{pycbc} \citep{pycbc} on the same uniform $(m_1, m_2)$ prior and Advanced-LIGO design PSD as the TaylorF2 warm-up. The frequency grid is extended to $f_{\rm max} = 2\,048$\,Hz to capture the merger and ringdown of the heaviest systems; the lower cutoff $f_{\rm low} = 20$\,Hz, frequency step $\Delta f = 0.5$\,Hz, whitening, and PCA pipeline (40 components fit on $5\,000$ noiseless waveforms, unit Gaussian noise added in the PCA basis) are identical to \cref{app:gw-details}.}

{\paragraph{Why the chirp-mass scaling fails.}
At the heavy end of the prior, the merger amplitude and ringdown frequency ($\propto 1/(m_1+m_2)$) carry a substantial fraction of the in-band signal-to-noise, so the $\mathcal{M}_c \propto (m_1 m_2)^{3/5}(m_1+m_2)^{-1/5}$ scaling that dominates inspiral-only Fisher information is no longer the principal degeneracy direction. The distillery exploits this empirically: the symbolic outputs are functions of $M = m_1+m_2$ and $\Delta m = m_1-m_2$ with no $\mathcal{M}_c$ factor, and we verified that running SR on $(M, \mathcal{M}_c)$ inputs reproduces a similar functional form, confirming the chirp-mass dependence is genuinely absent from the post-merger Fisher rather than hidden by a coordinate symmetry.}

{\paragraph{Calibration.}
The IMR sweep ($N_{\rm sims} \in \{500, 1\,000, 2\,500, 5\,000, 10\,000\}$) shows the learned $\eta$ coordinates achieve TARP coverage within the $1\sigma$ bootstrap band across the full $\alpha \in [0,1]$ range, whereas both $(m_1, m_2)$ and $(\mathcal{M}_c, q)$ NPEs under-cover at high $\alpha$.}

\label{lastpage}
\end{document}